\documentclass{article} %
\usepackage{iclr2027_conference,times}

\usepackage{amsmath,amsfonts,bm}

\def\eqref#1{equation~\ref{#1}}

\def\1{\bm{1}}

\DeclareMathAlphabet{\mathsfit}{\encodingdefault}{\sfdefault}{m}{sl}
\SetMathAlphabet{\mathsfit}{bold}{\encodingdefault}{\sfdefault}{bx}{n}

\usepackage{algorithm,algpseudocode}
\usepackage{amsthm}
\newtheorem{theorem}{Theorem}
\newtheorem{lemma}[theorem]{Lemma}
\usepackage{amssymb}

\newtheorem{corollary}{Corollary}
\newtheorem{remark}{Remark}
\usepackage{hyperref}
\usepackage{url}
\usepackage{cleveref}
\usepackage{graphicx}
\usepackage{xcolor}
\usepackage{booktabs}
\usepackage{multicol,multirow}
\usepackage{makecell}

\newcommand{\bz}{\mathbf{z}}
\newcommand{\bx}{\mathbf{x}}
\newcommand{\bs}{\mathbf{s}}

\newcommand{\edit}[1]{\textcolor{black}{#1}}

\title{FlowSGS: Improving Flow Priors for Inverse Imaging with Stochastic Interpolants}

\author{%
Tianao Li \\
Department of Computer Science \\
Northwestern University \\
NSF-Simons AI Institute for the Sky (SkAI) \\
\texttt{tianaoli@u.northwestern.edu}
\And
Xinhui Qian \\
Department of Statistics and Data Science \\
Northwestern University \\
\texttt{xinhuiqian2026@u.northwestern.edu}
\And
Emma Alexander \\
Department of Computer Science \\
Northwestern University \\
NSF-Simons AI Institute for the Sky (SkAI) \\
\texttt{ealexander@u.northwestern.edu}
}

\iclrfinalcopy %

\begin{document}

\maketitle
\lhead{Preprint.} %

\begin{abstract}
Flow matching has emerged as the state-of-the-art generative model and has been used for plug-and-play (PnP) priors to solve inverse problems in computational imaging.
However, existing flow-based inverse solvers make simplifying approximations in posterior sampling and/or assume linear forward models.
In contrast, principled diffusion-based methods exist but are computationally expensive.
To circumvent these problems, we introduce \textbf{FlowSGS}, a flow-based posterior sampling method using Split Gibbs Sampling (SGS) with annealed coupling, and make flow-specific interventions to the prior sampling and annealing steps, theoretically grounded in the Stochastic Interpolants (SI) framework. Using a novel timestep correction and accelerated warm-up with one-step flow models, our method is faster and more accurate than the state of the art across a range of imaging problems. We demonstrate the first use of flow-based samplers for a nonlinear inverse problem (Fourier phase retrieval) and apply one-step flow models to principled Bayesian inverse imaging.

\end{abstract}

\section{Introduction}
\label{sec:intro}

\vspace{-1mm}
Computational imaging systems aim to reconstruct an image $\mathbf{x} \in \mathbb{R}^n$ encoded in incomplete and noisy measurements $\mathbf{y} = \mathcal{A}(\mathbf{x}) + \mathbf{n} \in \mathbb{R}^m$, where $\mathcal{A}(\cdot)$ is the forward model and $\mathbf{n}$ represents noise.
These inverse problems are often ill-posed due to information loss, leading to uncertainty in the reconstructed image $\hat{\mathbf{x}}$.
To fully explore the solution space, it is often desired to sample from the Bayesian posterior
\begin{equation}
    \hat{\mathbf{x}} \sim p(\mathbf{x}|\mathbf{y}) \propto  p(\mathbf{y}|\mathbf{x})p(\mathbf{x}),
    \label{eq:inverse_problem}
\end{equation}
where $p(\mathbf{y}|\mathbf{x})$ is the data likelihood and $p(\mathbf{x})$ is the prior.
This is particularly important for scientific and medical imaging, where high-stakes decisions are being made in downstream applications.

\vspace{-1mm}
Deep generative models have significantly improved posterior sampling methods by implicitly capturing complex priors.
Specifically, score-based diffusion models~\citep{song2019generative,ho2020denoising,song2021scorebased} are capable of learning complex data distributions of images and videos, and have been widely used to learn expressive prior distributions $p(\mathbf{x})$.
Most diffusion-based posterior samplers condition on the measurements $\mathbf{y}$, either by taking gradient steps towards higher data likelihood (e.g., \citealp[DPS,][]{chung2023diffusion}; \citealp{jalal2021robust,kawar2022denoising,graikos2022diffusion,mardani2024a}), or by projecting onto the measurement space~\citep[e.g.,][]{song2022solving,chung2022score,chung2022improving}.
PnP-DM~\citep{wu2024principled} and several others~\citep{zhu2023denoising,xu2024provably,coeurdoux2024plug,chu2025split} leverage variable splitting to decompose the sampling into a likelihood step and a prior step, where a pretrained diffusion model is incorporated as a PnP prior.
InverseBench~\citep{zheng2025inversebench} provides a comprehensive benchmark of diffusion samplers on a variety of real-world inverse problems, where PnP-DM~\citep{wu2024principled} and DAPS~\citep{zhang2025improving} stand out as the state-of-the-art diffusion samplers, although they are computationally expensive due to solving the reverse-time SDE repeatedly. Other works~\citep{zhao2024cosign,garber2025zero,gulle2026pnp} leverage consistency models to increase efficiency.

\begin{figure*}[ht]
    \centering
    \includegraphics[width=0.7\linewidth]{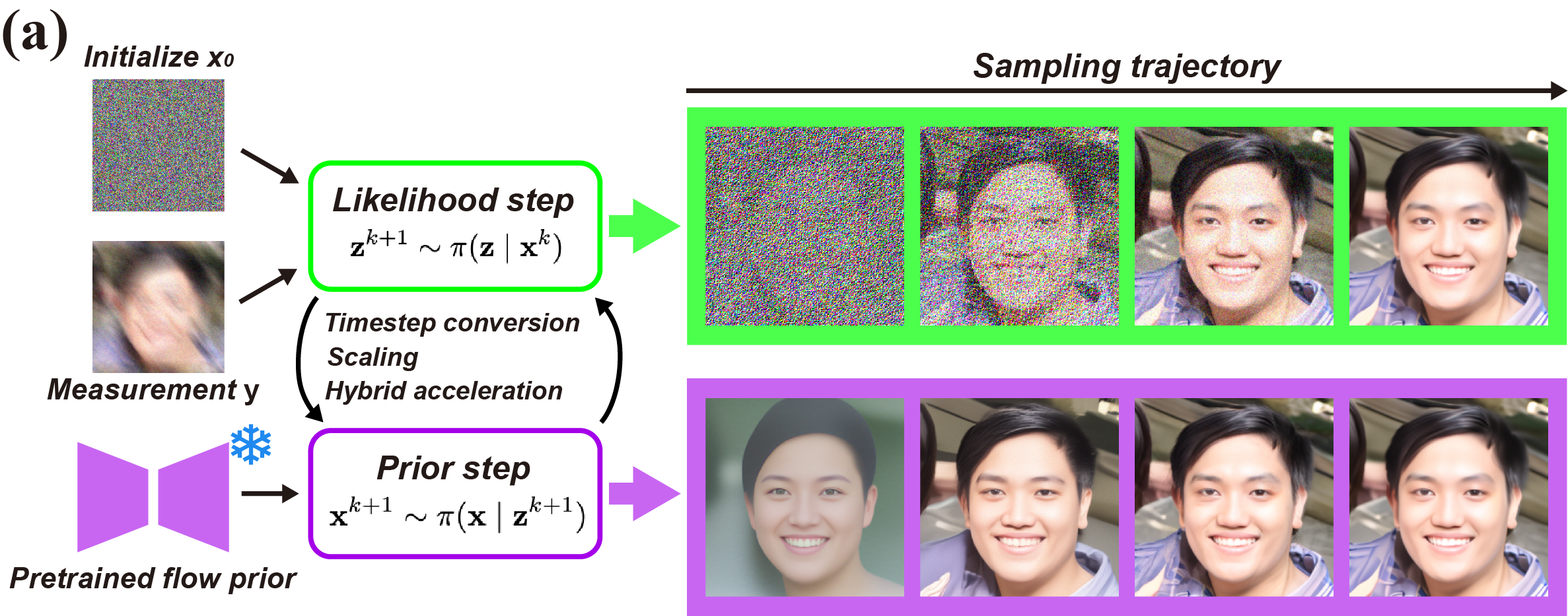}
    \hfill
    \includegraphics[width=0.29\linewidth]{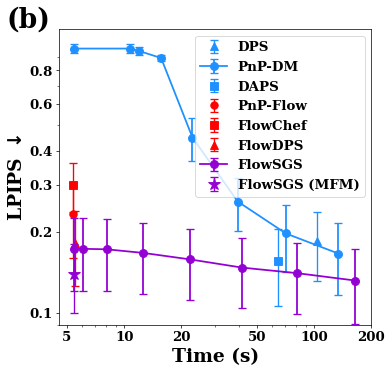}
    \caption{
        \textbf{A high-level overview of FlowSGS, a principled posterior sampler using flow matching as a plug-and-play prior.}
        (a) Our method iterates between a {\color{green} likelihood step} and a {\color{violet} prior step}, where the former uses proximal sampling or Langevin dynamics as appropriate to the problem, and the latter involves solving a reverse-time SDE using a pretrained flow prior.
       (b) Our method achieves high efficiency with flow priors, here averaged across 100 samples of motion deblurring.
    }
    \label{fig:teaser}
\end{figure*}

\vspace{-1mm}
Flow matching is an emerging generative technique that improves upon diffusion models by constructing straighter probability paths, which provide more efficient sampling, easier training, and higher quality data generation~\citep{lipman2023flow,liu2023flow}.
However, existing flow-based inverse solvers rely on mathematical approximations~\citep{pokle2024trainingfree,ben2024d,zhang2024flow,yan2025fig,martin2025pnpflow,patel2025flowchef,kim2025flowdps,askari2025latent} or assume a linear forward model~\citep{pokle2024trainingfree,zhang2024flow,yan2025fig,martin2025pnpflow,askari2025latent}.
Specifically, PnP-Flow~\citep{martin2025pnpflow} iterates between a data-fidelity gradient step, an interpolation step, and a denoising step that uses a flow network as a time-dependent denoiser.
FlowChef~\citep{patel2025flowchef} uses a gradient-free vector field to steer the ODE towards high measurement likelihood.
FlowDPS~\citep{kim2025flowdps} derives the flow version of Tweedie's formula, and estimates both the clean image and noise component to simultaneously maintain data-consistency and generative quality.
Each of these approaches relies on approximations that shift their posterior samples away from the Bayesian posterior.

\vspace{-1mm}
We present \textbf{FlowSGS}, a method that avoids the approximations used by current flow-based samplers through Split Gibbs Sampling (SGS) of the likelihood and prior steps, enabling us to target the actual Bayesian posterior. We adapt SGS to flow matching using the Stochastic Interpolants (SI) framework, which has revealed the deep connection between diffusion models' score functions and flow matching's velocity fields~\citep{albergo2025stochastic}.
A high-level overview of FlowSGS and its state-of-the-art performance and efficiency is shown in~\Cref{fig:teaser}.
Our experiments in~\Cref{sec:exp} show superior image reconstruction compared to both diffusion- and flow-based methods. We demonstrate accurate posterior sampling on a toy problem with a known posterior, and show high image quality on linear (deblurring, super-resolution, compressed-sensing MRI) and nonlinear (Fourier phase retrieval) tasks. Notably, FlowSGS is the first flow matching posterior sampler demonstrated on a nonlinear inverse problem (which is more sensitive to approximations in likelihood incorporation).
We also leverage one-step flow models~\citep{potaptchik2026meta} (see also \citealp{geng2025mean,geng2026improved,lu2026onestep}) to accelerate our solution without losing accuracy.

\vspace{-1mm}
\section{Foundational related work}
\label{sec:pre}

\vspace{-1mm}
\subsection{Stochastic interpolants}

\vspace{-1mm}
Stochastic interpolants~\citep[SI,][]{albergo2023building,albergo2025stochastic} provide a unified framework for diffusion models and flow matching, which both consider the process of interpolating between the data distribution $p(\mathbf{x}_0)$ and noise (usually a standard normal distribution $\mathcal{N}(0,\mathbf{I})$) for generative modeling.
Such interpolation processes have the form
 \begin{equation}
    \mathbf{x}_t = \alpha_t \mathbf{x}_0 + \sigma_t \boldsymbol{\epsilon}, \quad
    \boldsymbol{\epsilon} \sim \mathcal{N}(0,\mathbf{I}),
    \label{eq:si}
\end{equation}
where $(\alpha_t, \sigma_t)$ are a pair of smooth functions defined in
$t \in [0, 1]$, satisfying: (i) $\alpha_t^2 + \sigma_t^2 > 0$ for all $t \in [0,1]$; (ii) $\alpha_t$ and $\sigma_t$ are differentiable on $[0,1]$; and (iii) $\alpha_0 = \sigma_1 = 1, \alpha_1 = \sigma_0 = 0$.
Under these conditions, the stochastic process in~\Cref{eq:si}
interpolates without bias between clean data $p(\mathbf{x}_0)$ at $t=0$ and isotropic Gaussian noise $\mathcal{N}(0,\mathbf{I})$ at $t=1$.
To draw samples from the data distribution using the interpolant process, one can solve the reverse-time SDE~\citep{anderson1982reverse,albergo2025stochastic} given by
\begin{equation}
    \mathrm{d}\mathbf{x}_t = \left[\mathbf{v}(\mathbf{x}_t, t) - \frac{1}{2}w_t\mathbf{s}(\mathbf{x}_t, t)\right]\mathrm{d}t + \sqrt{w_t}\,\mathrm{d} \bar{\mathbf{w}}_t,
    \label{eq:reverse-sde}
\end{equation}
where $\mathbf{v}(\mathbf{x}, t) = \mathbb{E} \left[ \dot{\mathbf{x}}_t \vert \mathbf{x}_t = \mathbf{x} \right]$ is the velocity field, $\mathbf{s}(\mathbf{x},t) = \nabla_{\mathbf{x}} \log p_t(\mathbf{x})$ is the score function, and $w_t\geq 0$ is a diffusion coefficient picked at sampling time, and $\bar{\mathbf{w}}_t$ is the reverse-time Wiener process.
Different choices of the interpolant schedule $(\alpha_t, \sigma_t)$ have been used in diffusion models and flow matching (e.g., the VP-SDE used in score-based diffusion models~\citep{ho2020denoising,song2021scorebased}, the linear interpolant used in rectified flow~\citep{liu2023flow}).
We summarize common choices of the interpolant schedules $(\alpha_t, \sigma_t)$ and diffusion coefficients $w_t$ in Appendix~\ref{app:interpolant}.
While sampling with~\Cref{eq:reverse-sde} requires training both a flow model for the velocity field $\mathbf{v}(\mathbf{x}, t)$ and a diffusion model for the score function $\mathbf{s}(\mathbf{x}, t)$, prior work~\citep{albergo2025stochastic,ma2024sit} has shown that the velocity and the score convert from one to the other:
\begin{equation}
    \mathbf{s}(\mathbf{x}_t, t) = \frac{\alpha_t \mathbf{v}(\mathbf{x}_t, t) - \dot{\alpha}_t \mathbf{x}_t}{\sigma_t \gamma_t},
    \label{eq:score-from-velocity}
\end{equation}
where $\gamma_t \triangleq \dot\alpha_t \sigma_t - \alpha_t \dot\sigma_t < 0$.~\Cref{eq:score-from-velocity} allows sampling with the reverse-time SDE in~\Cref{eq:reverse-sde} using a single pretrained score or velocity model.
For instance, replacing $\mathbf{s}(\cdot,\cdot)$ with $\mathbf{v}(\cdot,\cdot)$ in~\Cref{eq:reverse-sde} yields
\begin{equation}
    \mathrm{d}\mathbf{x}_t
    = \left[\left(1 - \frac{w_t \alpha_t}{2\sigma_t \gamma_t}\right)\mathbf{v}(\mathbf{x}_t, t) + \frac{w_t \dot\alpha_t}{2\sigma_t \gamma_t}\,\mathbf{x}_t\right]\mathrm{d}t + \sqrt{w_t}\,\mathrm{d}\bar{\mathbf{w}}_t.
    \label{eq:sde_v}
\end{equation}
When the diffusion coefficient $w_t=0$,~\Cref{eq:sde_v} reduces to the probability-flow ODE,
which has been widely used in the sampling process of flow matching.

To enable one-step sampling, mean flows~\citep{geng2025mean,geng2026improved,lu2026onestep} learn the average velocity $\mathbf{u}(\mathbf{x},t,s)$ between any two timesteps $0 \leq t \leq s\leq 1$ of the probability-flow ODE.
Meta Flow Maps~\citep[MFM,][]{potaptchik2026meta} further learn a stochastic map $X_{s,u}(\boldsymbol{\epsilon};\mathbf{x}_t,t), \boldsymbol{\epsilon}\sim\mathcal{N}(0,\mathbf{I})$ to directly sample from the conditional posterior of the interpolant $p(\mathbf{x}_0\vert\mathbf{x}_t,t)$.

\vspace{-1mm}
\subsection{Split Gibbs sampling}

\vspace{-1mm}
Variable splitting has been widely applied to inverse problem solving, even before the era of deep generative priors.
Algorithms such as Half-Quadratic Splitting~\cite[HQS,][]{charbonnier1994two,geman1995nonlinear}, and the Alternating Direction Method of Multipliers~\cite[ADMM,][]{boyd2004convex,neal2011distributed} split the regularized MAP optimization into a likelihood step and a denoising step, where a pretrained denoiser can be plugged in~\citep{venkatakrishnan2013plug,chan2016plug,zhang2021plug,kamilov2023plug}.
Unrolled methods further turn these plug-and-play solvers into a neural network, and allow task-specific training of the plug-in denoisers~\citep{monga2021algorithm}.
However, these methods only produce a point estimate, rather than drawing random samples from the Bayesian posterior.

Split Gibbs Sampling~\citep[SGS,][]{vono2019split} provides a solution for this and has already seen applications in posterior sampling with classic priors~\citep{pereyra2023split}, pretrained denoisers~\citep{bouman2023generative,faye2024regularization}, or diffusion priors~\citep{xu2024provably,wu2024principled,chu2025split}.
SGS decomposes Bayesian posterior sampling in~\Cref{eq:inverse_problem} into a likelihood step and a prior step, and alternates between sampling from these two steps until convergence.
Specifically, in the $k$-th iteration, SGS has updates
\begin{align}
    \mathbf{z}^{k+1} &\sim \pi(\mathbf{z} \vert \mathbf{x}^{k}) \propto \exp\left[-f(\mathbf{z}; \mathbf{y}) - \frac{\|\mathbf{x}^{k} - \mathbf{z}\|^2}{2\rho_k^2}\right], && \triangleright~\text{\textcolor{green}{\bf Likelihood step}}
    \label{eq:llh}
    \\
    \mathbf{x}^{k+1} &\sim \pi(\mathbf{x} \vert \mathbf{z}^{k+1}) \propto \exp\left[-g(\mathbf{x}) - \frac{\|\mathbf{x} - \mathbf{z}^{k+1}\|^2}{2\rho_k^2}\right], && \triangleright~\text{\textcolor{violet}{\bf Prior step}}
    \label{eq:prior}
\end{align}
where $f(\mathbf{z};\mathbf{y}) = - \log p(\mathbf{y}|\mathbf{z})$ is the negative log data likelihood, $g(\mathbf{x}) = - \log p(\mathbf{x})$ is the negative log prior, and $\{\rho_k\}_{k=0}^{K-1}$ are augmented Lagrangian variables, optionally with an annealing schedule~\citep{wu2024principled}.
By introducing a modular structure in the sampling process, SGS solves the likelihood step and prior step independently, and simplifies the complexity of each subproblem.

\vspace{-1mm}
\section{Method}
\label{sec:method}

\vspace{-1mm}
FlowSGS uses Split Gibbs Sampling (SGS) to alternate between a \textcolor{green}{likelihood step} and a \textcolor{violet}{prior step} coupled by augmented Lagrangian parameters $\{\rho_k\}$.
We sample from the likelihood step using proximal sampling for its speed on linear $\mathcal{A}(\cdot)$, or Langevin dynamics for nonlinear $\mathcal{A}(\cdot)$ (see details in Appendix~\ref{app:details}), and use a user-defined exponential decay annealing strategy for the coupling parameter $\rho_k = \max(\rho_0 \alpha^k, \rho_{\text{min}})$ (see Appendix~\ref{app:annealing}).
When $\rho_{k}>\rho_{\min}$, the sampler goes into a \textit{warm-up} stage, which acts as an initialization heuristic that we find accelerates convergence. When $\rho_{\min}$ is reached, the sampler goes into a \textit{stationary} stage.
Among existing methods, this is most similar to PnP-DM~\citep{wu2024principled}, which leverages the same SGS-based problem decomposition but uses a diffusion model in its prior step. We also include a flow-specific annealing acceleration.

\begin{algorithm}[!b]
    \caption{Posterior sampling with FlowSGS.}
    \label{alg:flowsgs}
    \begin{algorithmic}[1]
        \Require measurements $\mathbf{y}$, forward model $\mathcal{A}(\cdot)$, velocity network $\mathbf{v}_\theta(\cdot,\cdot)$ with interpolant schedule $(\alpha_t, \sigma_t)$, diffusion coefficient $w_t$, iteration count $K$, coupling parameters $\{\rho_k\}_{k=0}^{K-1}$, flag \texttt{accel}, and Meta Flow Map $\mathbf{X}_{s,u}(\cdot\,;t,\mathbf{x}_t)$ sharing the schedule $(\alpha_t,\sigma_t)$ (if \texttt{accel})
        \State $\mathbf{x}^0 \sim \mathcal{N}(0,\mathbf{I})$ \Comment{Initialization.}
        \For {$k = 0, 1, \ldots, K-1$}
            \State $\mathbf{z}^{k+1} \sim \pi(\mathbf{z} \vert \mathbf{x}^{k}) \propto \exp\left[-f(\mathbf{z}; \mathbf{y}) - \frac{\|\mathbf{x}^{k} - \mathbf{z}\|^2}{2\rho_k^2}\right]$ \Comment{\textcolor{green}{\bf Proximal or Langevin dynamics.}} \label{line:proxlangdynamics}
            \State $t_k \leftarrow \left(\frac{\sigma_t}{\alpha_t}\right)^{-1}(\rho_k)$ \Comment{\textcolor{violet}{\bf Compute time $t_k$ corresponding to noise level $\rho_k$.}} \label{line:tk}
            \State $\mathbf{x}_{t_k} \leftarrow \alpha_{t_k}\, \mathbf{z}^{k+1}$ \Comment{\textcolor{violet}{\bf Scale $\mathbf{z}^{k+1}$ to match noise level at $t_k$.}} \label{line:zk}
            \If{\texttt{accel} \textbf{and} $\rho_k > \rho_{\min}$} \label{line:accelflag}
                \State $\boldsymbol{\epsilon}^k \sim \mathcal{N}(0,\mathbf{I})$ \label{line:epsk}
                \State $\mathbf{x}^{k+1} \leftarrow \mathbf{X}_{1-\varepsilon,\,\varepsilon}\!\left(\boldsymbol{\epsilon}^k\,;\,t_k,\,\mathbf{x}_{t_k}\right)$ \Comment{\textcolor{violet}{\bf One-step prior step sampling with MFM.}} \label{line:onestep}
            \Else %
                \State $\mathbf{x}^{k+1} \leftarrow \text{Euler–Maruyama}(\text{SDE},\mathbf{x}_{t_k},t_k,\epsilon)$
                \Comment{\textcolor{violet}{\bf Solve SDE with timestep correction.}} \label{line:SDE}
                \State \hspace{\algorithmicindent} where the SDE is
                $\mathrm{d}\mathbf{x}_t
                = \left[\left(1 - \frac{w_t \alpha_t}{2\sigma_t \gamma_t}\right)\mathbf{v}(\mathbf{x}_t, t) + \frac{w_t \dot\alpha_t}{2\sigma_t \gamma_t}\,\mathbf{x}_t\right]\mathrm{d}t + \sqrt{w_t}\,\mathrm{d}\bar{\mathbf{w}}_t.$ \label{line:SDE+1}
            \EndIf
        \EndFor
        \State \Return $\mathbf{x}^K$
    \end{algorithmic}
\end{algorithm}

\vspace{-1mm}
\paragraph{Prior step} We present several significant interventions in the prior step, expressed as Lines \ref{line:tk}--\ref{line:SDE+1} in~\Cref{alg:flowsgs}, which accommodate the integration of flow matching. Line~\ref{line:tk} computes the correct starting timestep $t_k$ for the prior step in the $k$-th iteration, adapted to the current noise level $\rho_k$ and the interpolant schedule of any pretrained flow model; see~\Cref{thm:denoising}.
Line~\ref{line:zk} scales the output of the likelihood step $\mathbf{z}^k$ such that its noise level matches the noise $\sigma_{t_k}$ in the interpolant at $t_k$.
Line~\ref{line:SDE} samples from the prior step by solving the reverse-time SDE with the Euler-Maruyama method~\citep{kloeden1977numerical}.
Line~\ref{line:SDE+1} shows the SI's reverse-time SDE using only the velocity field in the prior step, to our knowledge, used here for the first time in posterior sampling for inverse problems.
Note that the SDE discretization error is more important in flow-based methods than in diffusion models, because the relative sampling step count is greatly reduced.
We propose a novel timestep-correction technique to alleviate such error for fewer SDE steps.
Specifically, instead of naively starting at the closest discrete timestep to $t_k$, we use an adaptive stepsize in the first Euler step and start from the correct time $t_k$; see~\Cref{fig:ablation}.

\vspace{-1mm}
Lines \ref{line:accelflag}--\ref{line:onestep} introduce a hybrid acceleration to the prior step.
We draw intuition from the observation that the prior steps during the warm-up stage cost many more network evaluations than the stationary stage (see Appendix~\ref{app:interpolant}), and can be accelerated with one-step flow models.
However, such models are difficult to train~\citep{potaptchik2026meta}, and using them in the stationary stage leads to decreased image quality (see~\Cref{sec:ablation}).
We further observe that the prior step outputs $\mathbf{x}^k$ during warm-up (see~\Cref{fig:annealing}d) have low image quality compared to the stationary stage.
Therefore, it is feasible to use one-step samplers during warm-up for acceleration, but keep flow matching in the stationary stage for final image quality.
Such a \textit{hybrid} acceleration approach results in 1-NFE prior steps during all stages without the loss of accuracy~(\Cref{sec:exp}).
Importantly, we maintain convergence with the stationary stage unmodified (see~\Cref{sec:hybridprior}).

\vspace{-1mm}
\Cref{cor:collapse} specifies how our general framework can be reduced to PnP-DM with a specific choice of diffusion coefficient.
See more implementation details of our method in Appendix~\ref{app:details} and analysis of convergence in~\Cref{cor:convergence}.
\Cref{cor:warmup} establishes the validity of annealing acceleration through one-step sampling or other methods.
Our theory-driven improvements are shown to be empirically effective in our benchmark experiments in~\Cref{sec:results_linear}.
Our ablation studies in~\Cref{sec:ablation} confirm that each of our interventions on the SI-based prior step and hybrid one-step sampling is necessary.

\vspace{-1mm}
\paragraph{Likelihood step} Our approach to the likelihood step is also significant compared to existing flow-based methods.
We are the first flow-based method to incorporate measurement likelihood via SGS.
This change enables key applications in inverse imaging: more accurate posterior estimation (\Cref{sec:toy}) and nonlinear inverse problems (\Cref{sec:fpr}) with bimodal posteriors.
Many existing methods rely on simplifying approximations: some use linear projections between measurement and recovery space~\citep{pokle2024trainingfree,zhang2024flow,yan2025fig,martin2025pnpflow,askari2025latent}, while others incorporate the likelihood gradient as guidance~\citep{kim2025flowdps} or use gradient-free vector-field steering for measurement consistency~\citep{patel2025flowchef}.
In contrast, our method is a \textit{sampling} method that directly targets the actual Bayesian posterior.
See more discussions on \textit{guidance} and \textit{sampling} methods in~\cite{wu2024principled}.

\begin{figure*}[!t]
    \centering
    \includegraphics[width=\linewidth]{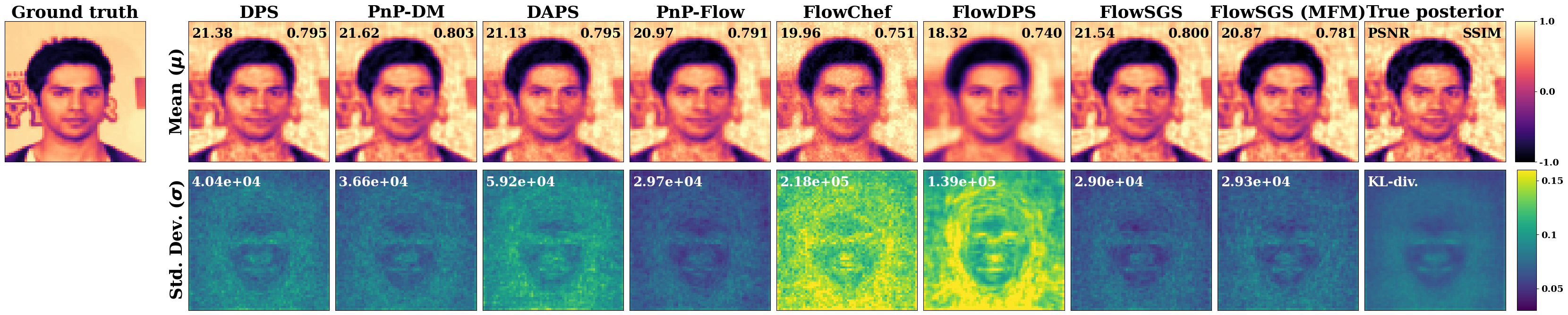}
    \caption{
        \textbf{Results on a toy compressed-sensing problem with an analytically available Gaussian posterior.}
        We show the mean image (row 1) and pixel-wise standard deviations (row 2) for the Gaussian posteriors.
        Our method estimates an accurate posterior with the lowest KL divergence compared to the true posterior. See Appendix~\ref{app:setup} for problem setup and~\Cref{sec:exp} for method settings.
    }
    \label{fig:toy_problem}
\end{figure*}

\vspace{-1mm}
\subsection{Prior step with flow matching}
\label{sec:prior}

\vspace{-1mm}
The key observation that drives the design of FlowSGS's prior step is the fact that sampling from the SGS prior step~\Cref{eq:prior} is equivalent to running SI's reverse-time SDE~\Cref{eq:reverse-sde} at the specific starting point related to the coupling parameter $\rho_k$, which we state below in~\Cref{thm:denoising}.

\begin{theorem}[Prior step equivalence to SI reverse-time SDE]
    Suppose $(\alpha_t, \sigma_t)$ defines an interpolant process with Gaussian probability path in~\Cref{eq:si}, and $\alpha_{t^*}>0$ and $t^*$ satisfies
    $\sigma_{t^*}/\alpha_{t^*}=\rho_k$, running the reverse-time SDE in~\Cref{eq:reverse-sde} from
    $\mathbf{x}_{t^*}=\alpha_{t^*}\bz^{k+1}$ to $t=0$ produces samples from
    the SGS prior step $\pi(\bx\vert\bz^{k+1})$ in~\Cref{eq:prior}, with the KL-optimal diffusion coefficient $w^{\mathrm{KL}}_t$.
    \label{thm:denoising}
\end{theorem}

\vspace{-1mm}
See proof of~\Cref{thm:denoising} in Appendix~\ref{app:proof_denoising}.
While PnP-DM~\citep{wu2024principled} shows a similar equivalence between its prior step and the diffusion's reverse-time SDE, our~\Cref{thm:denoising} generalizes to any Gaussian probability path defined by~\Cref{eq:si} under the SI framework.
In fact, we later show in~\Cref{sec:relation} that PnP-DM's prior step is a special case of our prior step.
The insights drawn from~\Cref{thm:denoising} directly lead to the timestep conversion, scaling, and reverse-SDE solving in Lines \ref{line:tk}--\ref{line:SDE+1} of~\Cref{alg:flowsgs}, which we will show to be effective in our ablation studies.

\vspace{-1mm}
\subsection{Relation to Existing PnP Methods}
\label{sec:relation}

\vspace{-1mm}
While in this paper we assume a flow prior, another way of applying the score--velocity conversion in~\Cref{eq:score-from-velocity} to the reverse-time SDE in~\Cref{eq:reverse-sde} is to replace the velocity field $\mathbf{v}(\cdot;\cdot)$ with the score function $\mathbf{s}(\cdot;\cdot)$, leading to a FlowSGS prior step using a diffusion model.
We show in the following corollary that FlowSGS's diffusion-based prior step with a specific choice of the diffusion coefficient $w_t$ results in the prior step of PnP-DM~\citep{wu2024principled}.
In other words, FlowSGS provides a general framework that can incorporate either a diffusion or flow prior, with arbitrary choices of using a score or velocity model, the interpolant process, and diffusion coefficients, and PnP-DM is a special case of FlowSGS under specific design choices.

\vspace{-1mm}
\begin{corollary}[PnP-DM as a special case of FlowSGS]
    Let $\lambda_t \triangleq -\gamma_t/\alpha_t
    = (\alpha_t\dot\sigma_t - \dot\alpha_t\sigma_t)/\alpha_t$.
    At the sampling stage, FlowSGS's timestep conversion and reverse-time SDE~(\ref{eq:reverse-sde}) reduce to PnP-DM (see (7) in~\cite{wu2024principled}) under specific conditions: the EDM diffusion coefficients
    $\alpha_t=s_{\mathrm{E}}(t)$, $\sigma_t=s_{\mathrm{E}}(t)\sigma_{\mathrm{E}}(t)$,
    the score--velocity conversion, %
    and the KL-optimal diffusion coefficient $w^{\mathrm{KL}}_t = 2\lambda_t \sigma_t$.
    \label{cor:collapse}
\end{corollary}
See proof of~\Cref{cor:collapse} in Appendix~\ref{app:proof_equivalence}.
In fact, such a connection is no surprise, since the stochastic interpolant framework also includes diffusion models in its mathematical formulations. Our contribution lies in specifying the exact form of this relationship.
FlowSGS's structural similarity to PnP-DM allows us to extend its convergence guarantee to FlowSGS under the SI parameterization, which we summarize in the following corollary.

\vspace{-1mm}
\begin{corollary}[Convergence analysis]
    Suppose \edit{that $\mathcal{A}$ is linear, that the likelihood step is computed with proximal sampling, that }the interpolant schedule $(\alpha_t,\sigma_t)$ satisfies that
    $\sigma_t/\alpha_t$ is strictly increasing on
    $(0,1]$, and the diffusion coefficient
    $w_t$ is bounded away from zero on the interval used by the
    prior step. %
    Then, under the $L^2$-accurate condition on the pretrained flow
    network,~\Cref{alg:flowsgs}
    with constant coupling $\rho_k\!\equiv\!\rho$ enjoys the same
    non-asymptotic convergence guarantee as PnP-DM~\cite[Theorem~3.1]{wu2024principled}: the
    average Fisher information between the non-stationary process driven by
    the velocity network and the stationary process driven by the true velocity field
    decays at a rate of $O(1/K)$ in the number of SGS iterations $K$, up to a
    floor set by the velocity network approximation error.
    Writing $\mu_k$ for the law of the $k$-th iterate, $\pi_\rho$ for the
    stationary distribution of the exact SGS chain at coupling $\rho$,
    $\delta=\inf\sqrt{w_t}$ over the prior-step interval, and $\varepsilon_v$ for
    the $L^2$ velocity-approximation error, the bound reads
    \begin{equation}\label{eq:bound}
    \frac{1}{K}\sum_{k=1}^{K}\mathrm{FI}(\mu_k\,\|\,\pi_\rho)
    \;\le\;
    \frac{C\,\mathrm{KL}(\mu_0\,\|\,\pi_\rho)}{K\min(\rho,\delta)^2}
    \;+\;
    E(\delta,\varepsilon_v),
    \end{equation}
    where $C$ is an absolute constant and $E(\delta,\varepsilon_v)$ is the
    velocity-approximation floor.
    \label{cor:convergence}
\end{corollary}

\vspace{-1mm}
See detailed proof of~\Cref{cor:convergence} in Appendix~\ref{app:convergence}.
Compared to PnP-DM, the assumptions become more flexible: PnP-DM inherits its diffusion-coefficient floor
from the EDM forward process (which can vanish as $t \to 0$ for iDDPM/EDM, requiring the workaround in~\cite[Appendix~A.3]{wu2024principled}),
whereas in FlowSGS, the diffusion coefficient $w_t$ is a free
sampling-time choice, so the floor condition is a constraint on the
user's choice rather than on the model.
The monotonicity condition on
$\sigma_t/\alpha_t$ is the SI analog of PnP-DM's Assumption~A.1 and is
satisfied by all standard schedules we consider.
Note that~\Cref{cor:convergence} describes the stationary stage when $\rho_k=\rho_{\min}$. We show in~\Cref{cor:warmup} that the warm-up stage does not affect the final convergence of the stationary chain, so that the varying $\rho$ in warm-up does not break convergence.
Note that the KL-optimal coefficient approaches 0 as $t\rightarrow 0$ for the linear interpolant ($w^{\mathrm{KL}}_t = 2t/(1-t)$) and GVP ($w^{\mathrm{KL}}_t = \pi\tan(\pi t/2)$), which violates the assumptions in~\Cref{cor:convergence}.
Following standard practice for diffusion-based samplers~\citep{zheng2025inversebench}, we avoid this irregularity by truncating the interpolant at $t \in [\epsilon, 1-\epsilon]$ with $\epsilon=1\times10^{-5}$ and terminating the reverse SDE at $t = \epsilon$ rather than $t = 0$, returning $\mathbf{x}^{k+1} = \mathbf{x}_\epsilon$.
Additionally, $w^{\mathrm{KL}}_t$ for both schedules is not regular as $t \to 1$. To avoid this, we follow~\cite{ma2024sit} and use regularization $w^{\mathrm{KL},\eta}_t=w^{\mathrm{KL}}_t\sqrt{\frac{\mathcal{L}_t}{\mathcal{L}_t+2\eta(w^{\mathrm{KL}}_t)^2}}$, where $\mathcal{L}_t$ is the velocity matching loss at time $t$ and $\eta$ is fine-tuned for each task.

\begin{figure*}[!t]
    \centering
    \includegraphics[width=\linewidth]{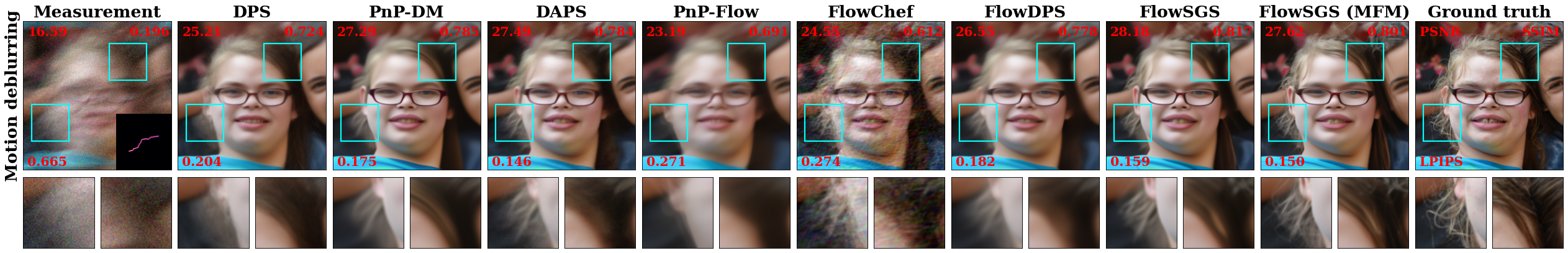}
    \includegraphics[width=\linewidth]{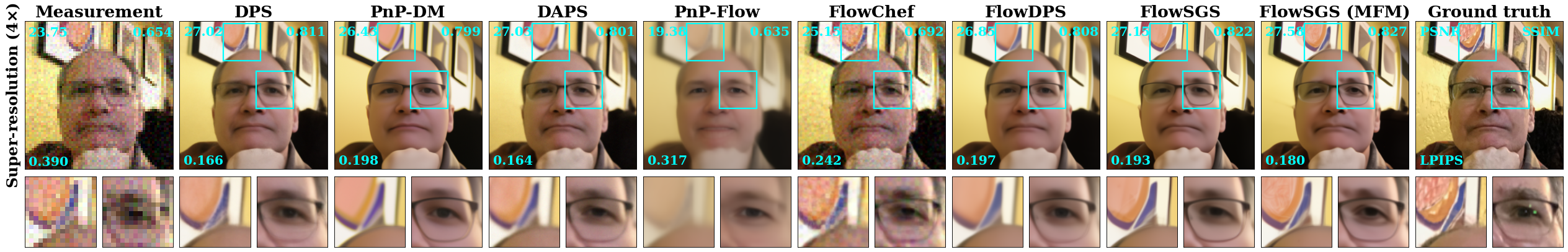}
    \includegraphics[width=\linewidth]{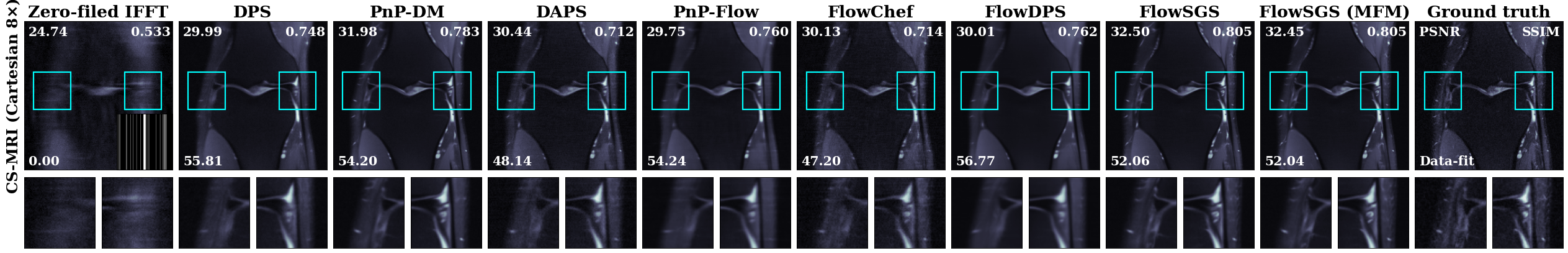}
    \caption{
        \textbf{Visual results on linear inverse problems.}
        We draw 8 posterior samples for each method and report the mean image.
        More results are shown in Appendix~\ref{app:results}.
    }
    \label{fig:linear}
\end{figure*}

\vspace{-1mm}
\subsection{Prior step with hybrid one-step flow models}
\label{sec:hybridprior}

\vspace{-1mm}
Our observation that the majority of the computation in FlowSGS is dominated by the warm-up annealing stage (see Appendix~\ref{app:interpolant}) led to our hybrid prior-step acceleration by integrating one-step flow models in the warm-up prior steps. We show that this does not affect convergence:
any non-pathological prior step sampler may be used during warm-up, including one-step flow-map models, and the convergence rate and approximation floor of the subsequent stationary stage are unaffected.

\begin{corollary}[Warm-up invariance]\label{cor:warmup}
Let $\mu_w$ denote the law of the prior-step output at the end of the warm-up stage.
Suppose the assumptions of~\Cref{cor:convergence} hold and $\mathrm{KL}(\mu_w\,\|\,\pi_{\rho_{\min}})<\infty$. Then the convergence properties of~\Cref{cor:convergence} apply to the stationary stage %
except that the bound in~\Cref{eq:bound} takes the new form
\begin{equation}
    \label{eq:bound_new}
    \frac{1}{K_s}\sum_{k=1}^{K_s}\mathrm{FI}(\mu_k\,\|\,\pi_{\rho_{\min}})
    \;\le\;
    \frac{C\,\mathrm{KL}(\mu_w\,\|\,\pi_{\rho_{\min}})}{K_s\min(\rho_{\min},\delta)^2}
    \;+\;
    E(\delta,\varepsilon_v),
\end{equation}
where $K_s$ is the number of steps in the stationary stage.
\end{corollary}

\vspace{-1mm}
See proof in Appendix~\ref{app:warmup}. Note that the warm-up affects only the bound, leaving the convergence rate and error floor unaffected, and that the assumption of finite KL divergence is extremely loose. Consequently, we see in practice that the multi-step prior step SDE can be replaced by a cheaper one-step model during annealing without weakening performance.

\vspace{-1mm}
\section{Experiments}
\label{sec:exp}

\vspace{-1mm}
We compare our method with three diffusion-based samplers: DPS~\citep{chung2023diffusion}, PnP-DM~\citep{wu2024principled}, DAPS~\citep{zhang2025improving}, as well as three flow-based samplers: PnP-Flow~\citep{martin2025pnpflow}, FlowChef~\citep{patel2025flowchef}, and FlowDPS~\citep{kim2025flowdps}.
We pretrain diffusion/flow models using the same NCSN++~\citep{song2021scorebased} architecture for fair comparisons (see details in Appendix~\ref{app:network}).
All flow-based methods use the linear interpolant schedule~\citep{liu2023flow} with 32 discretization steps by default. In our hybrid prior-step acceleration,  we use Meta Flow Maps~\citep[MFM,][]{potaptchik2026meta} in the prior steps (referred to as ``FlowSGS (MFM)'') during warm-up. %
We tune the hyperparameters of all methods for each task.

\vspace{-1mm}
\subsection{Toy problem}
\label{sec:toy}

\vspace{-1mm}
To validate that FlowSGS can accurately sample from the Bayesian posterior, we create a toy compressed sensing problem with a Gaussian likelihood and a Gaussian prior, such that the posterior is an analytically available Gaussian distribution.
See problem setups in Appendix~\ref{app:toy}.
We draw 128 posterior samples and compute the mean and pixel-wise standard deviation of the Gaussian posterior, and compute the PSNR and SSIM of the mean image with respect to the ground truth image, the KL divergence of the Gaussian posterior with respect to the true posterior, shown in~\Cref{fig:toy_problem}.
We show that FlowSGS and PnP-DM (both being \textit{sampling} methods) produce the most accurate posterior in terms of both mean image quality and KL divergence.
Note that PnP-Flow, although not claimed as a posterior sampler~\citep{martin2025pnpflow}, also produces an accurate posterior.
In contrast, FlowChef and FlowDPS significantly overestimate the uncertainty.

\vspace{-1mm}
\subsection{Linear inverse problems}
\label{sec:results_linear}

\begin{table}[!htb]
\footnotesize
\setlength{\tabcolsep}{2.5pt}
\centering
\caption{\textbf{Quantitative results averaged over 100 test samples.}}
\label{tab:results}
\begin{tabular}{lccccccccc}
\toprule
\multirow{2}{*}{\bf Method} & \multicolumn{3}{c}{\bf Motion deblurring} & \multicolumn{3}{c}{\bf Gaussian deblurring} & \multicolumn{3}{c}{\bf Super-resolution} \\
\cmidrule(lr){2-4} \cmidrule(lr){5-7} \cmidrule(lr){8-10}
& PSNR $\uparrow$ & SSIM $\uparrow$ & LPIPS $\downarrow$ & PSNR $\uparrow$ & SSIM $\uparrow$ & LPIPS $\downarrow$ & PSNR $\uparrow$ & SSIM $\uparrow$ & LPIPS $\downarrow$ \\ \midrule
DPS      & 26.76 & 0.777 & 0.184 & 28.54 & 0.810 & 0.151 & 28.63 & 0.819 & \underline{0.153} \\
PnP-DM   & 28.52 & 0.822 & 0.174 & 29.24 & 0.832 & 0.160 & 28.07 & 0.811 & 0.188 \\
DAPS     & 28.85 & 0.813 & 0.155 & 29.52 & 0.822 & 0.146 & 28.96 & 0.815 & 0.154 \\
PnP-Flow & 25.42 & 0.761 & 0.233 & 28.34 & 0.814 & 0.181 & 21.92 & 0.684 & 0.292 \\
FlowChef & 25.18 & 0.595 & 0.299 & 28.31 & 0.796 & 0.166 & 26.91 & 0.709 & 0.231 \\
FlowDPS  & 28.07 & 0.818 & 0.182 & 29.49 & 0.838 & 0.158 & 28.36 & 0.816 & 0.186 \\ \midrule
FlowSGS (VP)     & 28.93 & 0.840 & \underline{0.147} & \underline{29.58} & \underline{0.844} & \underline{0.139} & 28.50 & 0.824 & 0.157 \\
FlowSGS (GVP)    & \underline{29.01} & \underline{0.841} & 0.149 & 29.41 & 0.843 & 0.144 & 28.59 & 0.829 & 0.161 \\
FlowSGS (Linear) & \textbf{29.23} & \textbf{0.843} & 0.148 & \textbf{29.61} & \textbf{0.845} & 0.144 & \textbf{29.09} & \underline{0.840} & 0.159 \\
FlowSGS (MFM)    & 28.85 & 0.832 & \textbf{0.139} & 29.45 & 0.839 & \textbf{0.131} & \underline{29.06} & \textbf{0.841} & \textbf{0.142} \\
\bottomrule
\end{tabular}

\begin{tabular}{lccccccccc}
\toprule
\multirow{2}{*}{\bf Method} & \multicolumn{3}{c}{\bf CS-MRI (Cartesian, $8\times$)} & \multicolumn{3}{c}{\bf CS-MRI (radial, $8\times$)} & \multicolumn{3}{c}{\bf Fourier phase retrieval} \\
\cmidrule(lr){2-4} \cmidrule(lr){5-7} \cmidrule(lr){8-10}
& PSNR $\uparrow$ & SSIM $\uparrow$ & Data-fit $\downarrow$ & PSNR $\uparrow$ & SSIM $\uparrow$ & Data-fit $\downarrow$ & PSNR $\uparrow$ & SSIM $\uparrow$ & LPIPS $\downarrow$ \\ \midrule
DPS      & 30.79 & 0.802 & 55.503 & 34.09 & 0.850 & 52.301 & 15.40 & 0.416 & 0.484 \\
PnP-DM   & \underline{32.43} & 0.835 & 53.693 & 34.77 & \underline{0.865} & 50.808 & 32.13 & 0.901 & 0.104 \\
DAPS     & 31.92 & 0.819 & \underline{49.968} & 33.79 & 0.846 & 52.202 & 33.74 & 0.899 & \underline{0.070} \\
PnP-Flow & 29.76 & 0.804 & 53.980 & 33.91 & 0.856 & 52.477 & 24.63 & 0.720 & 0.221 \\
FlowChef & 29.83 & 0.752 & \textbf{47.151} & 33.30 & 0.817 & \textbf{48.896} & 18.24 & 0.414 & 0.433 \\
FlowDPS  & 29.09 & 0.783 & 56.554 & 34.02 & 0.854 & 54.141 & 27.55 & 0.808 & 0.183 \\ \midrule
FlowSGS (VP)     & 31.77 & 0.819 & 51.732 & 34.71 & 0.864 & \underline{49.758} & \underline{34.48} & \underline{0.925} & 0.082 \\
FlowSGS (GVP)    & 31.99 & 0.826 & 53.425 & \textbf{34.90} & \textbf{0.869} & 51.157 & 31.59 & 0.899 & 0.128 \\
FlowSGS (Linear) & 32.37 & \underline{0.842} & 51.930 & \underline{34.82} & \textbf{0.869} & 50.938 & \textbf{36.35} & \textbf{0.939} & \textbf{0.059} \\
FlowSGS (MFM)    & \textbf{32.57} & \textbf{0.844} & 51.924 & 34.81 & \textbf{0.869} & 50.916 & 14.88 & 0.474 & 0.584 \\
\bottomrule
\end{tabular}

\end{table}

\vspace{-1mm}
We benchmark the baselines and FlowSGS (with different interpolant schedules) over a variety of linear inverse problems (motion/Gaussian deblurring, 4$\times$ super-resolution, and compressed sensing MRI) using the FFHQ dataset~\citep{karras2019style} and the fastMRI knee dataset~\citep{zbontar2018fastmri} resized to 256$\times$256.
See problem setups in Appendix~\ref{app:exp_pixel}.
We pretrain diffusion/flow models on these datasets (checkpoints available).
In~\Cref{tab:results,tab:results_full}, we report metrics averaged over 100 test samples not used during training.
We draw 8 samples for each method and report the mean image.
Visual samples are shown in~\Cref{fig:linear}.
We show that FlowSGS stands out across a variety of inverse problems and is robust across different interpolant schedules.
See more results in Appendix~\ref{app:main}.

\begin{figure*}[!b]
    \centering
    \includegraphics[width=0.91\linewidth]{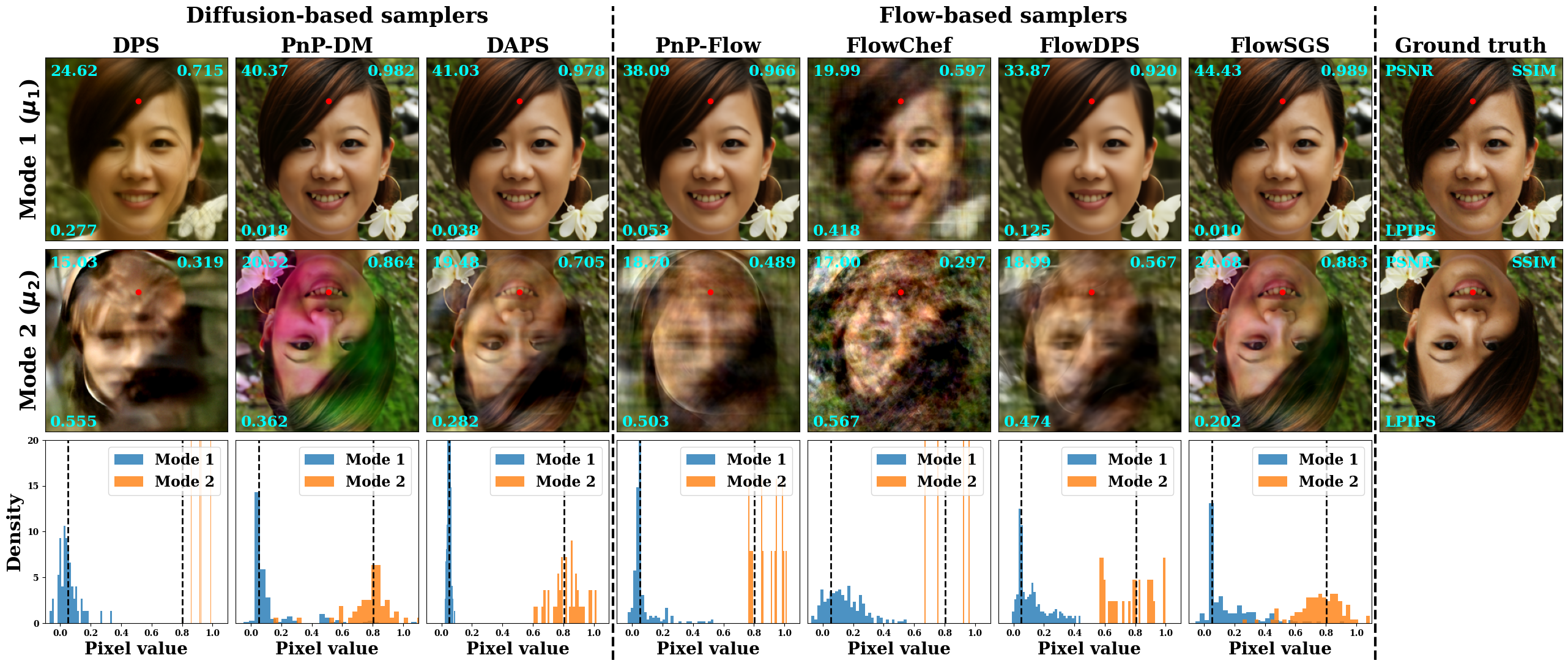}
    \caption{
        \textbf{Bimodal posterior for nonlinear Fourier phase retrieval.}
        With each method, we draw 512 samples and classify them into two modes.
        We compute the mean image using the top 25\% images for each mode, and show the histograms for a pixel in the red channel (marked by red dots).
    }
    \label{fig:fpr}
\end{figure*}

\vspace{-1mm}
\subsection{Nonlinear Fourier phase retrieval}
\label{sec:fpr}

\vspace{-1mm}
\edit{Though our mathematical guarantees apply only to linear problems, we show empirically that with appropriate modifications, nonlinear problems can be solved as well.} Fourier phase retrieval~\citep[FPR,][]{fienup1986phase,fienup2012understanding,shechtman2015phase} is an ill-posed inverse problem that widely exists in real-world imaging systems, where a bimodal posterior could exist due to the loss of Fourier phase information.
For the first time among flow-based samplers, we provide a benchmark on nonlinear FPR (see~\Cref{tab:results} and~\Cref{fig:fpr}).
Since FPR is very ill-posed and relies on good initialization, we characterize image reconstruction quality by reporting the single best sample among 8 posterior samples we draw, similar to~\cite{chung2023diffusion,wu2024principled} (see visual examples in~\Cref{fig:fpr_app}).
For uncertainty quantification, we follow~\cite{li2026shuffleflow} and draw 512 samples from each method and classify the samples into modes using the PSNR with respect to each mode (samples with PSNR$<$15 dB for both modes are discarded).
Using the best 25\% of the samples, we show the mean image of each mode in~\Cref{fig:fpr}a, and the histogram of a pixel in the red channel in~\Cref{fig:fpr}b.
While most methods reconstruct a bimodal posterior, we point out that our method reconstructs the best mean image for both modes (see metrics).
FlowSGS, PnP-DM, and DAPS are sampling methods that produce clear bimodal distributions centered at the true pixel value (vertical dashed lines in the histogram), while DPS, FlowChef, and FlowDPS, which rely on guidance or trajectory-steering rather than direct posterior sampling, recover corrupted rotated modes.
This demonstrates a strong advantage of principled posterior sampling over approximate methods in capturing multi-modal posteriors.

\vspace{-1mm}
To handle this challenging ill-posed problem, we use 100 SDE steps in the prior step and Langevin dynamics instead of proximal sampling for the likelihood step, resulting in a significant slowdown (see~\Cref{tab:time}). We also disable the hybrid prior-step acceleration, since the warm-up iterates of the non-convex FPR likelihood lie too far off the data manifold for a one-step flow map to recover, leading to corrupted image samples (see~\Cref{tab:results} and ~\Cref{fig:fpr_app}).

\begin{figure}
    \centering
    \includegraphics[width=0.37\linewidth]{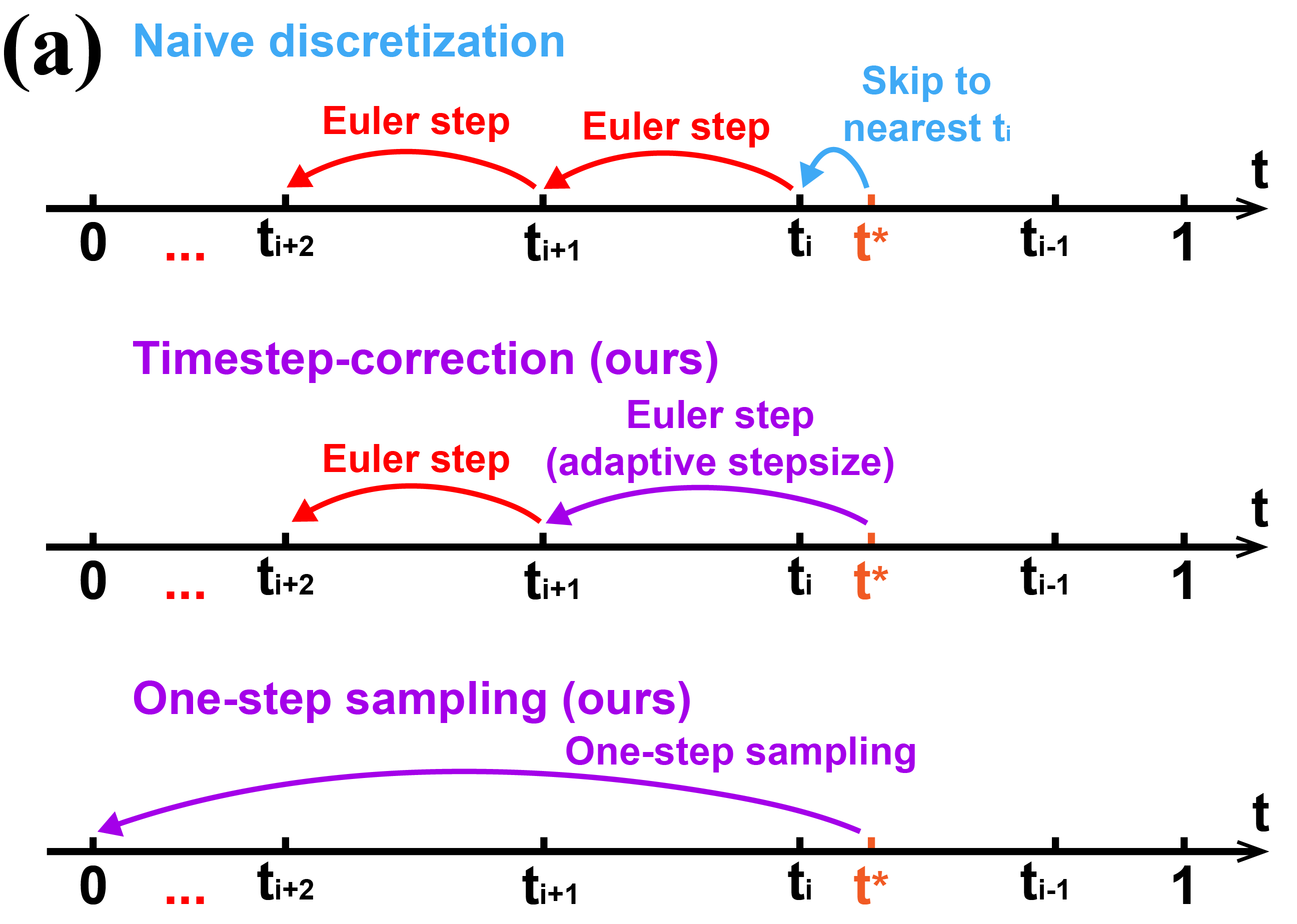}
    \hfill
    \includegraphics[width=0.29\linewidth]{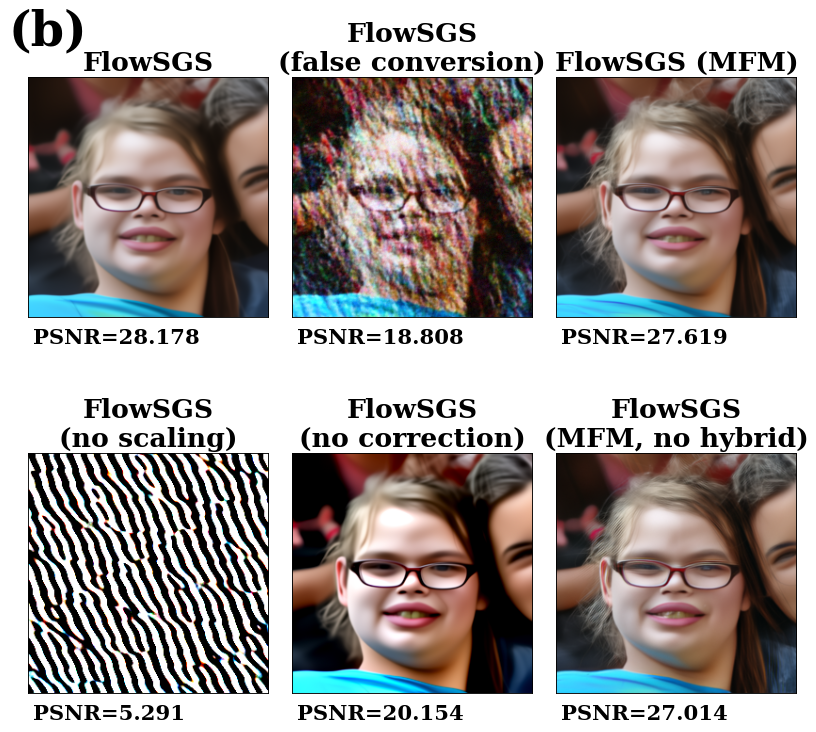}
    \hfill
    \includegraphics[width=0.28\linewidth]{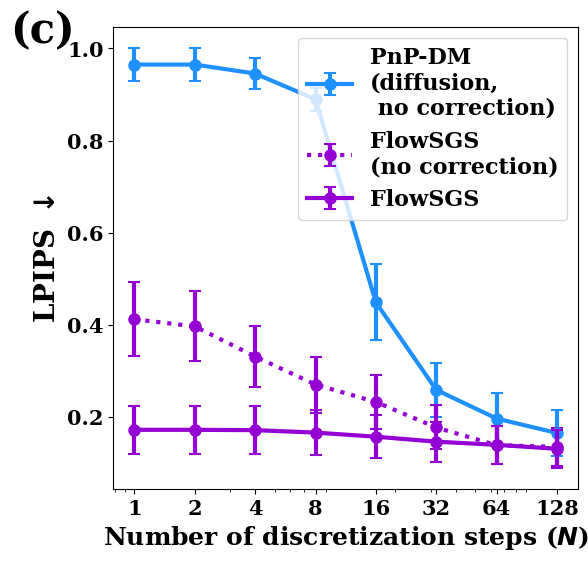}
    \caption{
        \textbf{Timestep-correction and one-step flow models enable efficient sampling.}
        (a) Illustration of FlowSGS's robust timestep correction to reduce discretization error.
        (b) Ablating our interventions decreases image quality on motion deblurring. See ground truth in~\Cref{fig:linear}.
        (c) Removing the timestep correction (dashed purple) limits reconstruction performance for lower discretization step counts $N$. For small $N$, the timestep correction improvement (dashed purple to purple) is on the same order as the diffusion-to-flow improvement (blue to dashed purple). All experiments here are done with motion deblurring, averaged over 100 test samples (error bars show standard deviations).
    }
    \label{fig:ablation}
\end{figure}

\vspace{-1mm}
\subsection{Ablation studies}
\label{sec:ablation}

\vspace{-1mm}
\paragraph{Timestep conversion, scaling, and timestep correction}
In~\Cref{fig:ablation}b and~\Cref{tab:ablation}, we show FlowSGS results on motion deblurring with missing components in~\Cref{alg:flowsgs}:
false conversion (using the VP schedule for a linear flow model in Line~\ref{line:tk}), no scaling (removing Line~\ref{line:zk}), no correction (using naive discretization in the SDE solver).
In~\Cref{fig:ablation}c (corresponding visual examples in~\Cref{fig:ablation_steps}), we observe that the timestep correction is particularly important at low discretization step counts. This sensitivity comes from the low NFE count of flow-based methods, which stands in contrast to the large NFE count required by diffusion-based methods. As a result, PnP-DM (blue) provides a useful comparison: its naive discretization step and reliance on diffusion reveal the relative importance of switching to flow matching and correcting the time step.

\vspace{-1mm}
\paragraph{Hybrid prior step acceleration}
We show in~\Cref{fig:ablation} and~\Cref{tab:ablation} that using MFM for the hybrid prior step acceleration (\Cref{fig:ablation}a) maintains image quality.
Using MFM in both stages (no hybrid) slightly decreases performance, since MFM is much more challenging to train and leads to decreased image quality if used in the stationary stage; see~\Cref{fig:ablation}b for visual examples.

\vspace{-1mm}
\paragraph{Interpolant schedules and diffusion coefficients}
In~\Cref{tab:results,tab:ablation_coeffs}, we compare FlowSGS results generated using three schedules and 5 diffusion coefficients (see Appendix~\ref{app:interpolant}) on motion deblurring.
We find that our method maintains performance under different settings.

\vspace{-1mm}
\section{Conclusion}
\label{sec:conclusion}

\vspace{-1mm}
We present modifications to SGS that accommodate SI-based generative models, which could even include yet-to-be-developed models.
In doing so, we outperform existing flow-based methods by targeting the exact Bayesian posterior. We also compare to principled diffusion-based methods, demonstrating improvements beyond direct replacement of diffusion with flow priors.
We provide theoretical grounding and state-of-the-art empirical results. %
Current major limitations include a reduction in speed for nonlinear inverse problems, opening a significant direction for future work. Additionally, we rely on pixel-space models which scale poorly with image size, and require tuning of a small set of hyperparameters ($\rho_0,\alpha,\rho_{\min},\sigma$ for linear problems, with two additional Langevin parameters for nonlinear problems).
Promising future directions include further exploration of annealing acceleration approaches and extending FlowSGS into semi-blind or blind samplers by incorporating our flow-based prior step into EM frameworks. %

\subsubsection*{Acknowledgments}
We gratefully acknowledge the support of the NSF-Simons AI Institute for the Sky (SkAI) via grants NSF AST-2421845 and Simons Foundation MPS-AI-00010513.
This material is based upon work supported by the U.S. National Science Foundation under Award No. 2542022.
This research used the Delta advanced computing and data resource, which is supported by the National Science Foundation (award OAC 2005572) and the State of Illinois. Delta is a joint effort of the University of Illinois Urbana-Champaign and its National Center for Supercomputing Applications.
The authors would like to thank Yu Sun, Deliang Wei, Rebecca Willett, and Tjitske Starkenburg for helpful discussions, and three anonymous reviewers for insightful feedback.

\bibliography{iclr2027_conference}
\bibliographystyle{iclr2027_conference}

\newpage
\appendix

\section{Theory}
\label{app:theory}

\subsection{Interpolant schedules and diffusion coefficients}
\label{app:interpolant}

In~\Cref{tab:interpolant_schedules}, we summarize three common interpolant schedules with which we experimented FlowSGS, as well as their score--velocity conversion~\Cref{eq:score-from-velocity} functions and time inversion functions.
In~\Cref{tab:coefficients}, we summarize common choices of diffusion coefficients $w_t$.
Note that $w_t^{\mathrm{KL}}$ is not regular at $t=1$ for Linear and GVP interpolants because $\alpha_t \to 0$; see~\cite{ma2024sit} for details.
We show the timestep conversion for each interpolant in~\Cref{fig:annealing}a.

\begin{table*}[h]
    \small
    \setlength{\tabcolsep}{2.4pt}
    \renewcommand{\arraystretch}{1.8}
    \centering
    \caption{
        {\bf Summary of interpolant schedules, score--velocity conversion (\ref{eq:score-from-velocity}), and time inversion.}
    }
    \label{tab:interpolant_schedules}
    \resizebox{\linewidth}{!}{%
    \begin{tabular}{lccc}
    \toprule
    \textbf{Property} & \textbf{VP-SDE} & \textbf{Linear} & \textbf{GVP} \\
    \midrule
    \textbf{Signal scale}
    & $\alpha_t = e^{-\frac{1}{2} \int_0^t \beta_s \mathrm{d}s}$
    & $\alpha_t = 1 - t$
    & $\alpha_t = \cos\!\left(\frac{\pi}{2}t\right)$
    \\
    \textbf{Noise scale}
    & $\sigma_t = \sqrt{1 - \alpha_t^2}$
    & $\sigma_t = t$
    & $\sigma_t = \sin\!\left(\frac{\pi}{2}t\right)$
    \\
    \makecell[l]{\textbf{Score-velocity} \\ \textbf{conversion}}
    & $\mathbf{s}=-\frac{2}{\beta_t} \mathbf{v}_t -\bx_t$
    & $\mathbf{s} = -\dfrac{(1-t)\mathbf{v}_t+\mathbf{x}_t}{t}$
    & $\mathbf{s}=-\frac{2}{\pi}\cot(\frac{\pi t}{2})\mathbf{v}_t -\bx_t$
    \\
    \textbf{Time inversion}
    & $t = B^{-1}\!\left(\log(1+\rho^2)\right)^*$
    & $t=\dfrac{\rho}{1+\rho}$
    & $t=\dfrac{2}{\pi}\arctan(\rho)$
    \\
    \midrule
    \makecell[l]{\textbf{Representative} \\ \textbf{models}}
    & \makecell[l]{DDPM~\citep{ho2020denoising}, \\  Score-SDE~\citep{song2021scorebased}}
    & \makecell[l]{Flow Matching~\citep{lipman2023flow}, \\
    Rectified Flow~\citep{liu2023flow}}
    & \makecell[l]{SiT~\citep{ma2024sit}}
    \\
    \bottomrule
    \end{tabular}}

    \vspace{0.5em}
    \begin{minipage}{0.94\linewidth}
    \footnotesize
    $^*$For VP-SDE, $\beta_t>0$ is the variance-preserving noise-rate schedule and
    $B(t)=\int_0^t \beta_s\,ds$ is its cumulative noise level. Therefore
    $B(t)=\log(1+\rho^2)$.
    \end{minipage}
\end{table*}

\begin{table}[htbp]
    \small
    \setlength{\tabcolsep}{5pt}
    \centering
        \caption{%
        \textbf{Diffusion coefficients $w_t$ for the reverse-time SDE},
        all chosen \emph{post-training} without affecting $\mathbf{v}_\theta$.
        ``Regular'' means finite and non-negative at the endpoint.
    }
    \label{tab:coefficients}
    \begin{tabular}{lccccc}
    \toprule
    & \textbf{ODE} & \textbf{Sigma} & \textbf{Sine} & \textbf{KL-optimal} & \textbf{Regularized KL-optimal} \\
    \midrule
    \textbf{Expression}
        & $0$
        & $\sigma_t$
        & $\sin^2(\pi t)$
        & $2\!\left(\dot{\sigma}_t\sigma_t - \dfrac{\dot{\alpha}_t\sigma_t^2}{\alpha_t}\right)$
        & $w^{\mathrm{KL},\eta}_t=w^{\mathrm{KL}}_t\sqrt{\frac{\mathcal{L}_t}{\mathcal{L}_t+2\eta(w^{\mathrm{KL}}_t)^2}}$ \\ [6pt]
    \textbf{Regular at $t=0$} & \checkmark & \checkmark & \checkmark & \checkmark & \checkmark \\
    \textbf{Regular at $t=1$} & \checkmark & \checkmark & \checkmark & $\times$ & \checkmark \\
    \bottomrule
    \end{tabular}
\end{table}

\begin{figure}[h]
    \centering
    \includegraphics[width=\linewidth]{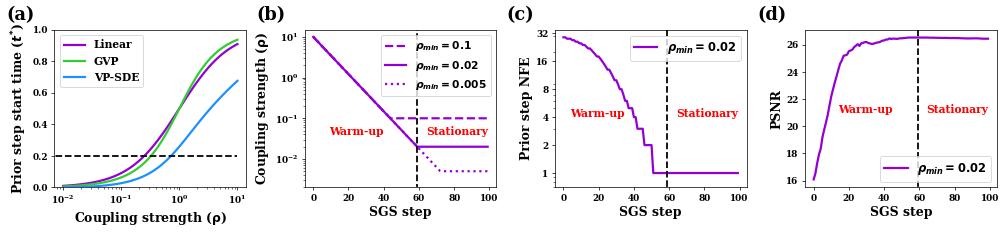}
    \caption{
        \textbf{Annealing of the coupling parameter $\rho$.}
        (a) Timestep conversion $t^*(\rho)$ for different interpolant schedules (see~\Cref{tab:interpolant_schedules}).
        (b) Annealing schedule for different $\rho_{\min}.$
        $\rho$ decreases during the warm-up stage and is fixed at $\rho_{\min}$ in the stationary stage.
        (c) Prior step NFE as a function of SGS step $k$. Most of the NFE is dominated by the warm-up stage.
        (d) PSNR of prior step output $\mathbf{x}^k$ as a function of SGS step $k$ on a motion deblurring test sample. The warm-up stage does not require high image quality and is therefore more robust to a one-step flow sampler with slightly degraded image quality.
        We label the warm-up and stationary stage for $\rho_{\min}=0.02$ with dashed vertical lines in b,c,d.
    }
    \label{fig:annealing}
\end{figure}

\subsection{Proof of~\texorpdfstring{\Cref{thm:denoising}}{Theorem 1}}
\label{app:proof_denoising}

\begin{proof}[Proof of~\Cref{thm:denoising}]
The SGS prior step targets
\begin{equation}
    \pi(\bx_0\vert\bz)
    \propto
    \exp[-g(\bx_0)]
    \exp\!\left(
        -\frac{\|\bx_0-\bz\|^2}{2\rho_k^2}
    \right),
\end{equation}
which has the form of a Bayesian denoising posterior with prior
$p(\bx_0)\propto \exp[-g(\bx_0)]$ and Gaussian likelihood $\pi(\bz\vert\bx_0) = \mathcal{N}(\bz;\bx_0,\rho_k^2\mathbf{I})$.
This implies that $\bz\vert\bx_0 \sim \mathcal{N}(\bz;\bx_0,\rho_k^2\mathbf{I})$.

Suppose $t^*$ satisfies
$\sigma_{t^*}/\alpha_{t^*}=\rho_k$ and $\alpha_{t^*}>0$.
Solving the reverse-time SDE in~\Cref{eq:reverse-sde} from $t^*$ to $0$ follows the conditional probability
\begin{equation}
    \begin{aligned}
        \pi(\bx_0\vert\bx_{t^*}) &\propto \pi(\bx_{t^*}\vert\bx_0)p(\bx_0)  \\
        &= \mathcal{N}(\bx_{t^*};\alpha_{t^*}\bx_0,\sigma_{t^*}^2\mathbf{I})p(\bx_0) \\
        & \propto \exp\!\left(-\frac{ \|\bx_{t^*}-\alpha_{t^*}\bx_0\|^2 }{2\sigma_{t^*}^2}\right) \exp[-g(\bx_0)] \\
        &= \exp\!\left(-\frac{ \|\bx_{t^*}/\alpha_{t^*}-\bx_0\|^2 }{2(\sigma_{t^*}/\alpha_{t^*})^2}\right) \exp[-g(\bx_0)] \\
        &=  \exp\!\left(-\frac{ \|\bx_{t^*}/\alpha_{t^*}-\bx_0\|^2 }{2\rho_k^2}\right) \exp[-g(\bx_0)],
    \end{aligned}
    \label{eq:denoise_proof}
\end{equation}
where the first line uses Bayes' rule, and the second line uses the definition of the Gaussian probability path in~\Cref{eq:si}.
\Cref{eq:denoise_proof} indicates that the conditional probability $\pi(\bx_0\vert\bx_{t^*})$ is equal to the prior step $\pi(\bx_0\vert\bz)$ when
$\bx_{t^*}=\alpha_{t^*}\bz$.
Therefore, running the SI's reverse-time SDE with initialization at $t^*$ and $\bx_{t^*}=\alpha_{t^*}\bz$, the distribution we get at $\bx_0$ is exactly the distribution of the SGS prior step $\pi(\bx_0\vert\bz)$.
This provides the foundation of our prior step in~\Cref{alg:flowsgs}, where the time conversion (Line~\ref{line:tk}) follows
\begin{equation}
    t^*(\rho)=\left(\frac{\sigma_{t^*}}{\alpha_{t^*}}\right)^{-1}(\rho),
\end{equation}
the scaling (Line~\ref{line:zk}) of the output of the likelihood step $\bz$ follows
\begin{equation}
    \bx_{t^*}=\alpha_{t^*}\bz.
\end{equation}
These two steps are essential preparation steps in service for the prior step for exact noise level matching.
Note that with common choices of interpolants, $\left(\frac{\sigma_{t^*}}{\alpha_{t^*}}\right)$ is monotonic such that there always exists a unique solution $t^* \in [0,1]$ for any coupling parameter $\rho>0$.

\end{proof}

\subsection{Proof of~\texorpdfstring{\Cref{cor:collapse}}{Corollary 1}}
\label{app:proof_equivalence}

PnP-DM~\citep{wu2024principled} uses EDM diffusion~\citep{karras2022elucidating} to formulate the diffusion process:
\begin{equation}
    \mathbf{x}_t = s_{\mathrm{E}}(t)
    \bigl( \bx_0+\sigma_{\mathrm{E}}(t)\boldsymbol{\epsilon} \bigr).
\end{equation}
While under the stochastic-interpolant parameterization, we have
\begin{equation}
    \mathbf{x}_t
    =
    \alpha_t\bx_0+\sigma_t\boldsymbol{\epsilon}.
\end{equation}
These two parameterizations have the relation
\begin{equation}
    s_{\mathrm{E}}(t)=\alpha_t,
    \qquad
    \sigma_{\mathrm{E}}(t)=\frac{\sigma_t}{\alpha_t}.
    \label{eq:param_relation}
\end{equation}

Define
\begin{equation}
    \gamma_t
    =
    \dot\alpha_t\sigma_t-\alpha_t\dot\sigma_t,
    \qquad
    \lambda_t
    =
    -\frac{\gamma_t}{\alpha_t} =
    \frac{\alpha_t\dot\sigma_t-\dot\alpha_t\sigma_t}{\alpha_t}.
\end{equation}

\begin{lemma}\label{lem:edm-si-deriv}
Under the parameter correspondence in~\Cref{eq:param_relation}, we have
\begin{equation}
    \dot\sigma_{\mathrm{E}}(t)
    =
    \frac{\lambda_t}{\alpha_t}.
\end{equation}
\end{lemma}

\begin{proof}
Since $\sigma_{\mathrm{E}}(t)=\sigma_t/\alpha_t$,
\begin{align}
    \dot\sigma_{\mathrm{E}}(t)
    = \frac{
        \dot\sigma_t\alpha_t-\sigma_t\dot\alpha_t
    }{
        \alpha_t^2
    }
    =
    -\frac{
        \dot\alpha_t\sigma_t-\alpha_t\dot\sigma_t
    }{
        \alpha_t^2
    }
    =
    -\frac{\gamma_t}{\alpha_t^2} =
    \frac{\lambda_t}{\alpha_t}.
\end{align}
\end{proof}

Now, we can prove that the PnP-DM reverse-time SDE under the EDM parameterization
\citep{wu2024principled,karras2022elucidating}
\begin{equation}\label{eq:pnpdm-eq7}
    \mathrm{d}\mathbf{x}_t
    =
    \left[
        \frac{\dot{s_{\mathrm{E}}}(t)}{s_{\mathrm{E}}(t)}\mathbf{x}_t
        -
        2s_{\mathrm{E}}(t)^2
        \dot\sigma_{\mathrm{E}}(t)
        \sigma_{\mathrm{E}}(t)
        \nabla_{\mathbf{x}_t}
        \log p
        \!\left(
            \frac{\mathbf{x}_t}{s_{\mathrm{E}}(t)};
            \sigma_{\mathrm{E}}(t)
        \right)
    \right]\mathrm{d}t
    +
    s_{\mathrm{E}}(t)
    \sqrt{
        2\dot\sigma_{\mathrm{E}}(t)
        \sigma_{\mathrm{E}}(t)
    }\,
    \mathrm{d}\mathbf{w}_t,
\end{equation}
is a special case of the FlowSGS reverse-time SDE.

\begin{proof}[Proof of~\Cref{cor:collapse}]
First, by replacing the EDM diffusion coefficients using~\Cref{eq:param_relation}, we have
\begin{align}
    \frac{\dot{s_{\mathrm{E}}}(t)}{s_{\mathrm{E}}(t)}
    &=
    \frac{\dot\alpha_t}{\alpha_t}, \\
    2s_{\mathrm{E}}(t)^2\dot\sigma_{\mathrm{E}}(t)\sigma_{\mathrm{E}}(t)
    &=
    2\alpha_t^2
    \left(
        \frac{\lambda_t}{\alpha_t}
    \right)
    \left(
        \frac{\sigma_t}{\alpha_t}
    \right) =
    2\lambda_t\sigma_t, \\
    s_{\mathrm{E}}(t)^2
    2\dot\sigma_{\mathrm{E}}(t)\sigma_{\mathrm{E}}(t)
    &=
    \alpha_t^2
    2
    \left(
        \frac{\lambda_t}{\alpha_t}
    \right)
    \left(
        \frac{\sigma_t}{\alpha_t}
    \right) =
    2\lambda_t\sigma_t.
\end{align}
This shows that the PnP-DM reverse-time SDE has drift coefficient
$\dot\alpha_t/\alpha_t$, score coefficient $2\lambda_t\sigma_t$, and
diffusion coefficient squared $2\lambda_t\sigma_t$.

The velocity-score relation can be written as
\begin{equation}
    \mathbf{v}(\mathbf{x}_t,t)
    =
    \frac{\dot\alpha_t}{\alpha_t}\mathbf{x}_t
    -
    \lambda_t\sigma_t\mathbf{s}(\mathbf{x}_t,t).
\end{equation}
Substituting this into the SI reverse-time SDE in~\Cref{eq:reverse-sde} gives
\begin{equation}
    \begin{aligned}
        \mathrm{d}\mathbf{x}_t
        &=
        \left[
            \frac{\dot\alpha_t}{\alpha_t}\mathbf{x}_t
            -
            \lambda_t\sigma_t\bs(\mathbf{x}_t,t)
            -
            \frac{w_t}{2}\bs(\mathbf{x}_t,t)
        \right]\mathrm{d}t
        +
        \sqrt{w_t}\,\mathrm{d}\bar{\mathbf{w}}_t \\
        &=
        \left[
            \frac{\dot\alpha_t}{\alpha_t}\mathbf{x}_t
            -
            \left(
                \lambda_t\sigma_t+\frac{w_t}{2}
            \right)
            \bs(\mathbf{x}_t,t)
        \right]\mathrm{d}t
        +
        \sqrt{w_t}\,\mathrm{d}\bar{\mathbf{w}}_t .
    \end{aligned}
\end{equation}
Choosing $w_t = 2\lambda_t\sigma_t$ yields
\begin{equation}
    \mathrm{d}\mathbf{x}_t
    =
    \left[
        \frac{\dot\alpha_t}{\alpha_t}\mathbf{x}_t
        -
        2\lambda_t\sigma_t
        \bs(\mathbf{x}_t,t)
    \right]\mathrm{d}t
    +
    \sqrt{2\lambda_t\sigma_t}\,
    \mathrm{d}\bar{\mathbf{w}}_t.
\end{equation}
This matches PnP-DM's prior step reverse-time SDE.
In FlowSGS, the initialization at $t^*$ follows $\bx_{t^*}=\alpha_{t^*}\bz^{k+1}$, which directly translates to PnP-DM's initialization since $\alpha_{t^*}=s_{\mathrm{E}}(t^*)$.
In summary, PnP-DM is recovered as the special case of FlowSGS with equivalent reverse-time SDE and initializations, when $w_t=2\lambda_t\sigma_t$.

For VP-SDE, $\lambda_t\sigma_t=\beta_t/2$, so this diffusion coefficient aligns with and recovers the reverse process in VP-SDE.
\begin{equation}
    w_t^{\mathrm{KL}}
    =
    2\lambda_t\sigma_t
    =
    \beta_t.
\end{equation}
\end{proof}

\begin{remark}
    EDM fixes the reverse-SDE diffusion coefficient to the value implied by the
    forward diffusion process $w_t=s_{\mathrm{E}}(t) \sqrt{2\dot\sigma_{\mathrm{E}}(t)\sigma_{\mathrm{E}}(t)}$. In contrast, the SI framework allows
    any $w_t\geq0$ at sampling time: $w_t=0$ gives the probability-flow ODE,
    $w_t=2\lambda_t\sigma_t$ recovers PnP-DM, and other choices can be tuned
    without retraining.
\end{remark}

\subsection{Proof of~\texorpdfstring{\Cref{cor:convergence}}{Corollary 2}}
\label{app:convergence}

Rather than reproduce the full Fokker--Planck calculation, we show that
FlowSGS satisfies all the structural conditions used in
Theorem~3.1 of PnP-DM~\citep{wu2024principled}, so that the same proof
applies after a parameter substitution. We highlight three aspects
and then describe the substitution.

\paragraph{Identical SGS skeleton.}
The likelihood step of FlowSGS (\Cref{alg:flowsgs}, Line~\ref{line:proxlangdynamics}) is identical to PnP-DM's: both target the same conditional distribution
$\pi(z\vert x)\propto\exp[-f(z;y)-\|x-z\|^2/2\rho^2]$ via proximal sampling or Langevin
dynamics. PnP-DM's Brownian-bridge interpretation
(\citealp[Proposition~A.2]{wu2024principled}) and the resulting KL/Fisher
contraction (\citealp[Lemma~2]{vempala2019rapid}) therefore carry
over to FlowSGS.

\paragraph{Prior step has the canonical drift--diffusion form.}
PnP-DM's key technical lemma (\citealp[Lemma~A.4]{wu2024principled}) is a
generic Fokker--Planck argument that bounds the time-derivative of the
KL divergence between any pair of diffusion processes that share the
same diffusion coefficient $c(t)$ and differ only in their drifts. It
makes no assumption about \emph{how} the drift is constructed.

The FlowSGS prior-step SDE~\Cref{eq:sde_v} is exactly such a process: its drift is an affine combination of the velocity $\mathbf v(\mathbf x,t)$ and the state $\mathbf x$, and its diffusion coefficient is $\sqrt{w_t}$. The non-stationary process (driven by the pretrained $\mathbf{v}_\theta$) and the stationary process (driven by the true $\mathbf v$) share the same $c(t)=\sqrt{w_t}$, which is the only structural requirement of the lemma. Hence PnP-DM's lemma
applies to FlowSGS's prior step verbatim, with the squared drift gap
proportional to $\|\mathbf v_\theta-\mathbf v\|_2^2$ instead of
$\|\mathbf s_\theta-\mathbf s\|_2^2$.

\paragraph{SI-equivalent assumptions.}
PnP-DM's Theorem~3.1 relies on two assumptions:
\begin{itemize}
  \item \emph{(Assumption A.1)} $\sigma_{\mathrm{E}}(t)$ is strictly increasing, so there is a unique $t^\star$ with $\sigma_{\mathrm{E}}(t^\star)=\rho$. Under the SI parameter correspondence $\sigma_{\mathrm{E}}(t)=\sigma_t/\alpha_t$ (\Cref{cor:collapse}), this becomes the requirement that the noise-to-signal ratio $\sigma_t/\alpha_t$ is strictly increasing---precisely the well-posedness condition of our~\Cref{thm:denoising}, and satisfied by all schedules in \Cref{tab:interpolant_schedules}.
  \item \emph{(Diffusion-coefficient floor)} The infimum $\delta:=\inf_{[0,t^\star]}s_E(t)\sqrt{2\dot\sigma_E\sigma_E}$ of PnP-DM's reverse-SDE diffusion coefficient is strictly positive. Under the same correspondence this infimum becomes $\inf_{[0,t^\star]}\sqrt{w_t}$, where $w_t$ is FlowSGS's sampling-time diffusion coefficient. This is where the assumptions become \emph{looser}: PnP-DM's $\delta$ is fixed by the EDM forward process and can vanish as $t \to 0$ (e.g.\ for iDDPM/EDM), forcing the workaround in~\cite[Appendix~A.3]{wu2024principled}; FlowSGS lets the user pick $w_t$ freely at sampling time, so the floor is a constraint on the user's choice rather than on the model.
\end{itemize}
The score-accuracy assumption of PnP-DM translates to an
$L^2$-accuracy condition on the velocity field, which
is equivalent to the score condition via the velocity--score relation
(\Cref{eq:score-from-velocity} in the main paper) up to bounded scalar factors.

\paragraph{Reduction.}
Combining the three observations above, every step of the proof of Theorem~3.1 in~\cite[Appendix~A.2]{wu2024principled} goes through: the likelihood-step contraction is unchanged; the prior-step
contraction follows from \cite[Lemma~A.4]{wu2024principled} applied to
the canonical-form SDE described in the \emph{canonical drift--diffusion form} paragraph above, with PnP-DM's $c(t)$
replaced by $\sqrt{w_t}$ and the score-error term replaced by the
corresponding velocity-error term; the telescoping over $K$ outer
iterations is identical. The resulting non-asymptotic bound has the
same shape as PnP-DM's: an $O(1/K)$ contraction from initialization
governed by $\min(\rho,\delta)^2$, plus a velocity-approximation floor
governed by $\delta^2$. Collecting the constants from that argument gives the bound stated in~\Cref{eq:bound}.
\qed

\subsection{Proof of~\texorpdfstring{\Cref{cor:warmup}}{Corollary 3}}
\label{app:warmup}

Write $\pi\equiv\pi_{\rho_{\min}}$ for the SGS target at coupling $\rho_{\min}$,
and let $P$ denote the Markov kernel of one iteration of
Algorithm~\ref{alg:flowsgs} in the stationary stage ($\rho_k\equiv\rho_{\min}$) with the reverse-time
SDE prior step (Lines~\ref{line:SDE}--\ref{line:SDE+1}). The bound in~\Cref{cor:convergence} is
obtained, through the reduction in Appendix~\ref{app:convergence} to
\citep[Theorem~3.1]{wu2024principled}, by telescoping a per-iteration KL descent
inequality: for any law $\nu$ with $\mathrm{KL}(\nu\,\|\,\pi)<\infty$, one
iteration of $P$ decreases the KL divergence to $\pi$ by a multiple of the
Fisher information, up to an additive velocity-approximation error. The
contraction constant is set by $\min(\rho_{\min},\delta)^2$ and the error term
by the $L^2$-accuracy of $\mathbf{v}_\theta$ and the floor $\delta$; both are
properties of $P$ alone, and the inequality holds for any input law with finite
KL divergence to $\pi$, regardless of how that law was produced.

The warm-up stage is a Markov kernel $Q$ (a composition of $K_w$ iterations at couplings $\rho_k>\rho_{\min}$ with any prior-step sampler), so that $\mu_w=\mu_0 Q$ and the stationary-stage iterates have laws $\mu_w P^{k}$, $k=1,\dots,K_s$.
Since $\mathrm{KL}(\mu_w\,\|\,\pi)<\infty$ by assumption, the descent inequality applies to each of the $K_s$ stationary iterations starting from $\mu_w$.
Summing, the KL terms telescope; using $\mathrm{KL}\ge0$, one obtains the bound of Corollary~\ref{cor:convergence} with $K$ replaced by $K_s$ and the initial KL divergence replaced by $\mathrm{KL}(\mu_w\,\|\,\pi)$; see~\Cref{eq:bound_new}. The constants, the floor, and the approximation error are those of $P$ and are therefore identical to Corollary~\ref{cor:convergence}.
The warm-up kernel $Q$ enters nowhere else.
\qed

\section{Inverse problem setups}
\label{app:setup}

\subsection{Toy problem}
\label{app:toy}

In our toy problem, we used images from the CelebA dataset~\citep{liu2015faceattributes}\footnote{License unknown.}.
All images are converted to grayscale, resized to 64$\times$64, and normalized to $[-1, 1]$ before any calculations.
We generate a random Gaussian matrix $\mathbf{A} \in \mathbb{R}^{1024\times4096}$ as the compressed sensing matrix.
The forward model is $\mathbf{y} = \mathbf{A}\mathbf{x} + \boldsymbol{\epsilon}$, where $\boldsymbol{\epsilon} \sim \mathcal{N}(0, \sigma^2\mathbf{I})$ and $\sigma = 0.01$.
For the prior, we sample 50,000 male face images from the CelebA dataset and fit a multivariate Gaussian distribution $\mathcal{N}(\boldsymbol{\mu}_0, \mathbf{\Sigma}_0)$ to it.
Along with the Gaussian likelihood, this gives us a Gaussian posterior
\begin{equation}
    p(\mathbf{x}\vert\mathbf{y}) = \mathcal{N}(\boldsymbol{\mu}, \mathbf{\Sigma}), \quad \boldsymbol{\mu} = \mathbf{\Sigma} \left( \mathbf{\Sigma}_0^{-1} \boldsymbol{\mu}_0 + \frac{1}{\sigma^2} \mathbf{A}^\top \mathbf{y} \right), \quad \mathbf{\Sigma} = \left( \mathbf{\Sigma}_0^{-1} + \frac{1}{\sigma^2} \mathbf{A}^\top \mathbf{A} \right)^{-1}
\end{equation}
In~\Cref{fig:toy_setup}, we illustrate the measurement, prior, and analytical posterior used in the toy problem.

\begin{figure}[h]
    \centering
    \includegraphics[width=0.75\linewidth]{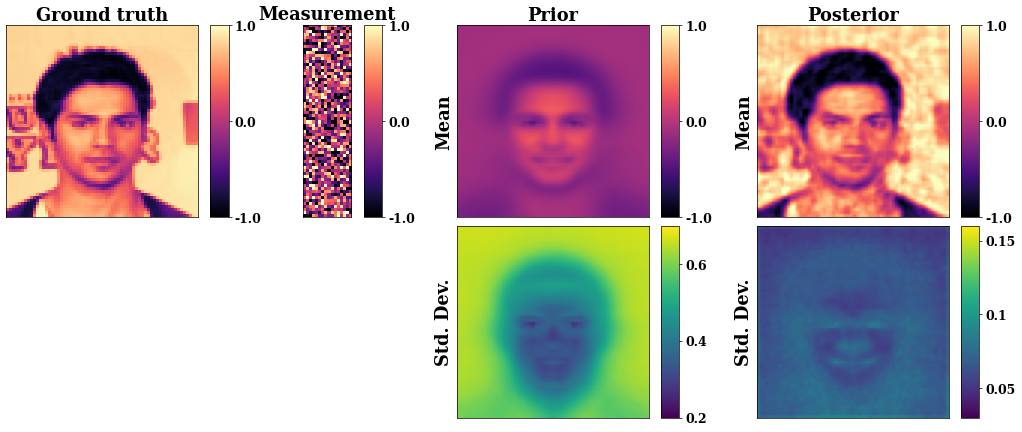}
    \caption{\bf Setup of the toy compressed sensing problem.}
    \label{fig:toy_setup}
\end{figure}

\subsection{Pixel-space model experiments}
\label{app:exp_pixel}

In our experiments with pixel-space models (\Cref{sec:results_linear}), we use the FFHQ dataset~\citep{karras2019style}\footnote{Creative Commons BY-NC-SA 4.0 license.} and fastMRI single-coil knee dataset~\citep{zbontar2018fastmri}.
All images are resized to 256$\times$256 and normalized to $[0, 1]$ as ground truths, which are then used to generate degraded measurements for different tasks.

\paragraph{Motion deblurring and Gaussian deblurring}
The forward model is $\mathbf{y} = \mathbf{H}\mathbf{x} + \boldsymbol{\epsilon}$, where $\mathbf{H}$ is a convolution operator with a blur kernel
and $\boldsymbol{\epsilon} \sim \mathcal{N}(0, \sigma^2\mathbf{I})$ with $\sigma = 0.05$.
For motion deblurring, we generate 64$\times$64 motion blur kernels using this code\footnote{\url{https://github.com/LeviBorodenko/motionblur} (license unknown).} with trajectory intensity $0.5$ and then pad them to 256$\times$256.
For Gaussian deblurring, we directly generate 256$\times$256 blur kernels with standard deviation $\sigma_k = 3.0$ pixels.
All kernels are normalized to have a sum of 1.0.

\paragraph{Super-resolution}
The degradation process is modeled as $\mathbf{y} = \mathbf{S}\mathbf{x} + \boldsymbol{\epsilon}$,
where $\mathbf{S}$ is a $4{\times}$ average pooling downsampling operator,
and $\boldsymbol{\epsilon} \sim \mathcal{N}(0, \sigma^2\mathbf{I})$ with $\sigma = 0.05$.

\paragraph{Compressed sensing MRI}
The forward model is $\mathbf{y} = \mathbf{M}\mathcal{F}(\mathbf{x}) + \boldsymbol{\epsilon}$, where $\mathbf{M}$ is the masking operator, $\mathcal{F}(\cdot)$ is the 2D Fourier transform, and $\boldsymbol{\epsilon} \sim \mathcal{N}(0, \sigma^2\mathbf{I})$ with $\sigma = 0.02$.
We generate both Cartesian and radial sampling masks with an acceleration rate of 8.0.
The center 8\% of the components are fully sampled with the Cartesian mask, while the other lines are randomly sampled.
Lines in the radial mask are uniformly sampled.

\paragraph{Fourier phase retrieval}
The forward model is $\mathbf{y} = \vert\mathcal{F}(\mathbf{P}\mathbf{x})\vert + \boldsymbol{\epsilon}$, where $\mathbf{P}$ is the padding operator with an oversampling rate of 8.0 (padding the 256$\times$256 images to 512$\times$512), $\mathcal{F}(\cdot)$ is the 2D Fourier transform, and $\boldsymbol{\epsilon} \sim \mathcal{N}(0, \sigma^2\mathbf{I})$ with $\sigma = 0.01$.
The oversampling operation is used to increase the number of measurements to reduce the ill-posedness of the problem~\citep{wu2024principled}.

\section{Implementation details}
\label{app:details}

\subsection{Network architectures}
\label{app:network}

In our toy problem and benchmarking experiments with pixel-space models (\Cref{sec:toy,sec:results_linear}), we pretrain our own diffusion and flow models using the NCSN++ architecture~\citep{song2021scorebased} implemented by InverseBench~\citep{zheng2025inversebench}.
For Meta Flow Maps, we add additional zero-initialized conditional embedding layers to NCSN++ and initialize the weights from the velocity checkpoint before training.
All models are trained for 500 epochs with the Adam optimizer on each dataset and use an exponential moving average (EMA) for better performance.
For the diffusion models, we use the VP-SDE as the diffusion schedule, while for the flow models, we train with the VP-SDE, linear, and GVP schedules (see summary in~\Cref{tab:interpolant_schedules}) and compare them in our ablation studies (see~\Cref{tab:results}).
Within every task, all diffusion/flow models are trained on the same training set and share the same architecture and optimization parameters for fair comparisons.
Our model checkpoints trained on CelebA, FFHQ, and fastMRI knee are available on our project page.

\subsection{Likelihood step proximal sampling}
\label{app:proxiaml}

We assume a Gaussian noise model $\mathbf{n} \sim \mathcal{N}(0,\sigma^2 \mathbf{I})$, and the likelihood step in~\Cref{eq:llh} follows the form:
\begin{equation}
    \mathbf{z} \sim \pi(\mathbf{z} \vert \mathbf{x}^k, \mathbf{y})
    \propto \exp\!\left[
        -\frac{\|\mathcal{A}(\mathbf{z}) - \mathbf{y}\|^2}{2\sigma^2}
        - \frac{\|\mathbf{z} - \mathbf{x}^k\|^2}{2\rho_k^2}
    \right],
    \label{eq:likelihood-dist}
\end{equation}
where $\sigma$ represents the standard deviation for the noise model.
When the forward model  is linear ($\mathcal{A}(\mathbf{x})=\mathbf{Ax}$), there exists an analytical proximal sampling solution to the likelihood step~\citep{wu2024principled}:
\begin{equation}
    \bz \sim \mathcal{N}(\mathbf{m}(\mathbf{x}^k),\mathbf{\Lambda}^{-1})
\end{equation}
where $\mathbf{\Lambda} := \frac{1}{\sigma^2}\mathbf{A}^{\top}\mathbf{A} + \frac{1}{\rho_k^2} \mathbf{I} \quad \text{and} \quad \mathbf{m}(\mathbf{x}) := \mathbf{\Lambda}^{-1} \left(\frac{1}{\sigma^2}\mathbf{A}^{\top} \mathbf{y} + \frac{1}{\rho_k^2} \mathbf{x} \right)$.
We use such proximal sampling for all our linear inverse problems.
For each task, we tune $\sigma$ for the best reconstruction quality.

\subsection{Likelihood step with Langevin dynamics}
\label{app:lgvd}

When the forward model $\mathcal{A}(\cdot)$ is nonlinear (e.g., Fourier phase retrieval), the proximal generator does not exist.
In this case, we sample from~\Cref{eq:likelihood-dist} using Langevin dynamics~\citep{uhlenbeck1930theory}, which leads to an increase in sampling time.
We use the implementation of Langevin dynamics from InverseBench~\citep{zheng2025inversebench}.
For each Langevin step, we compute the gradient:
\begin{equation}
    \nabla E(\mathbf{z}) =
    \frac{1}{2\sigma^2} \nabla_{\mathbf{z}} \|\mathcal{A}(\mathbf{z}) - \mathbf{y}\|^2
    + \frac{\mathbf{z} - \mathbf{x}^k}{\rho_k^2},
    \label{eq:langevin_gradient}
\end{equation}
and update $\mathbf{z}$ via:
\begin{equation}
    \mathbf{z}^{(\ell+1)} = \mathbf{z}^{(\ell)}
    - \eta \,\nabla E\!\left(\mathbf{z}^{(\ell)}\right)
    + \sqrt{2\eta}\;\boldsymbol{\epsilon}, \qquad
    \boldsymbol{\epsilon} \sim \mathcal{N}(0,\mathbf{I}),
    \label{eq:langevin-update}
\end{equation}
where $\eta$ denotes the Langevin step size.
We use 100 steps in the Langevin sampling in our Fourier phase retrieval experiments.

\subsection{Annealing schedule}
\label{app:annealing}

We use 100 iterations in SGS and summarize the annealing hyperparameters used in each task for FlowSGS in~\Cref{tab:annealing}.
We discovered that the sampling quality is most sensitive to $\rho_{\min}$, which determines the stationary coupling of the likelihood and prior step, as well as the starting point of the ending prior step that refines the final outputs.
A reasonably large starting level $\rho_0$ and a decay rate $\alpha>=0.9$ lead to similar performance.

\begin{table}[htbp]
    \small
    \centering
    \caption{\textbf{Hyperparameters of the annealing schedule for different tasks.}}
    \label{tab:annealing}
    \begin{tabular}{lccc}
        \toprule
        \bf Inverse problem & Starting level ($\rho_0$) & Minimum level ($\rho_{\min}$) & Decay rate ($\alpha$) \\
        \midrule
        Toy problem & 10 & 0.05 & 0.9 \\
        Motion deblurring & 10 & 0.05 & 0.9 \\
        Gaussian deblurring & 10 & 0.05 & 0.9 \\
        Super-resolution ($4\times$) & 10 & 0.05 & 0.9 \\
        CS-MRI (Cartesian, $8\times$) & 10 & 0.05 & 0.9 \\
        CS-MRI (radial, $8\times$) & 10 & 0.02 & 0.9 \\
        Fourier phase retrieval & 10 & 0.01 & 0.9 \\
        \bottomrule
    \end{tabular}
\end{table}

\section{Additional results}
\label{app:results}

\begin{figure}
    \centering
    \includegraphics[width=\linewidth]{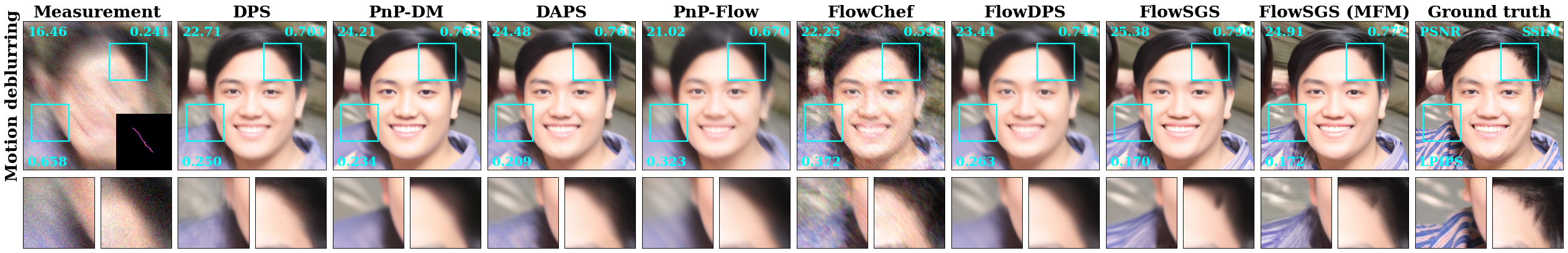}
    \includegraphics[width=\linewidth]{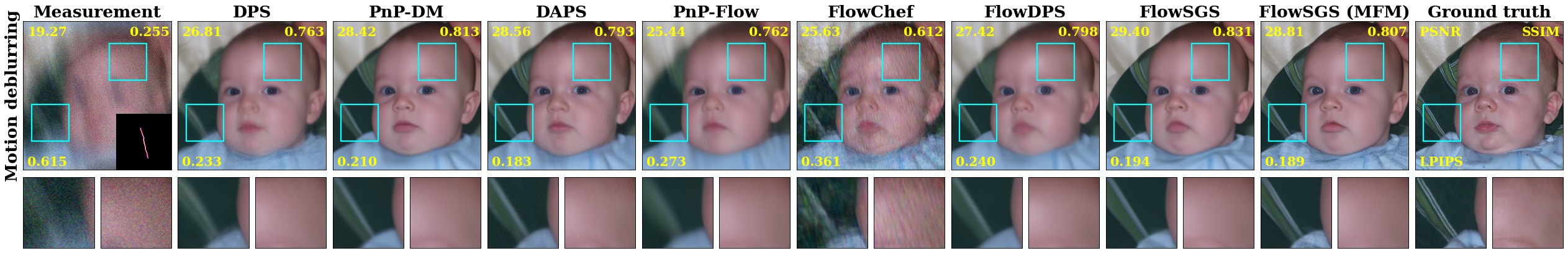}
    \includegraphics[width=\linewidth]{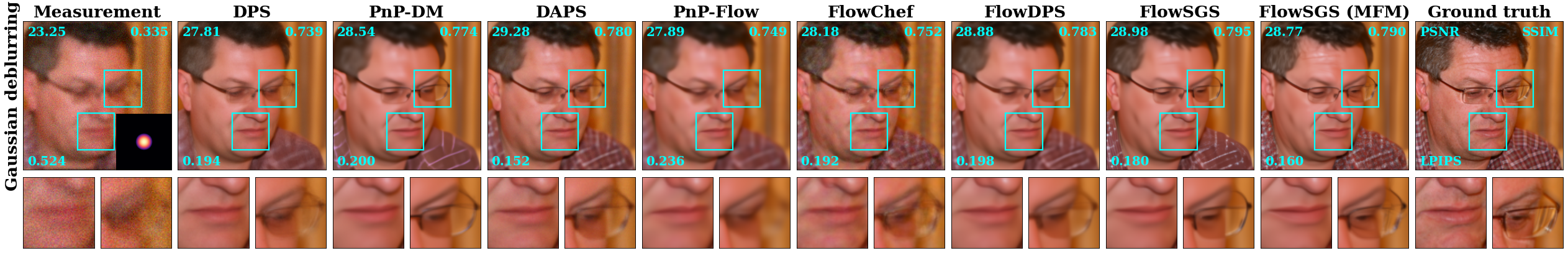}
    \includegraphics[width=\linewidth]{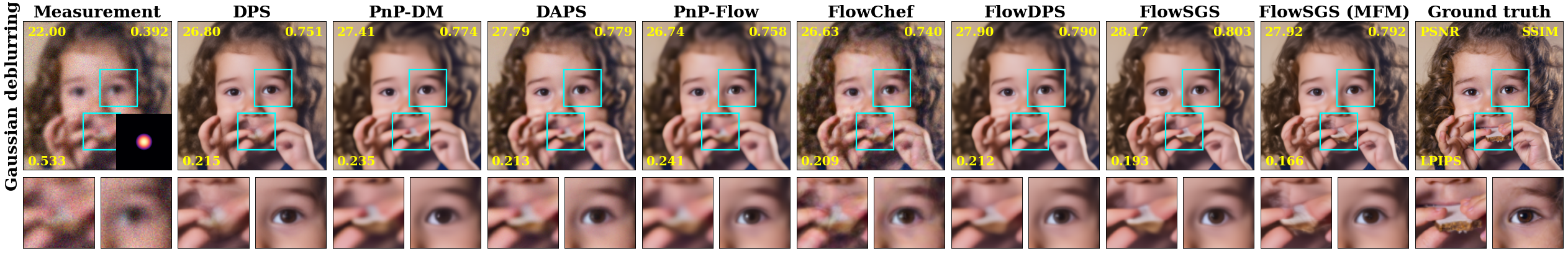}
    \includegraphics[width=\linewidth]{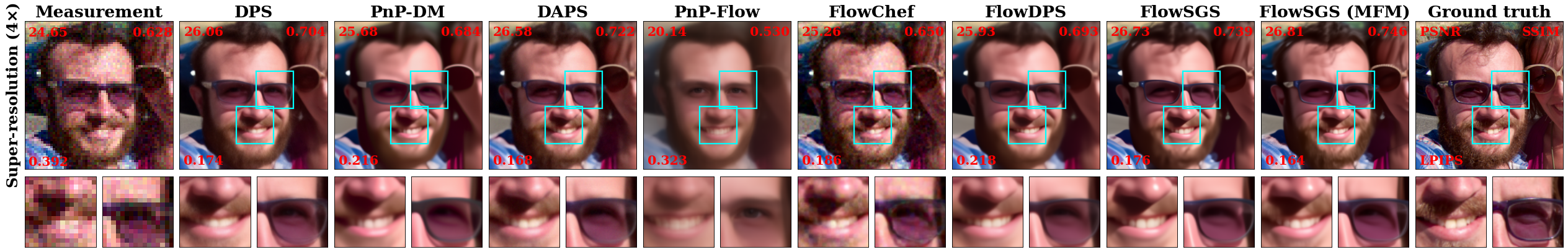}
    \includegraphics[width=\linewidth]{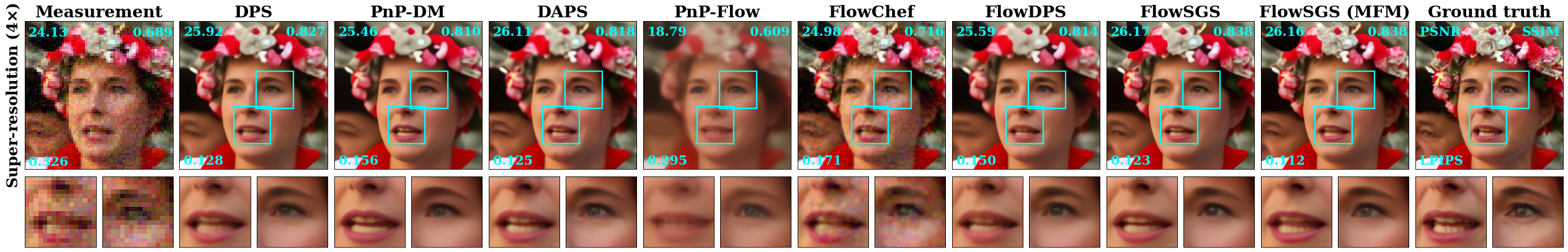}
    \caption{
        \textbf{Additional visual samples on motion deblurring, Gaussian deblurring, and 4$\times$ super-resolution.}
        We draw 8 posterior samples for each method and show the mean image.
    }
    \label{fig:pixel_results_supp}
\end{figure}

\begin{figure}
    \centering
    \includegraphics[width=\linewidth]{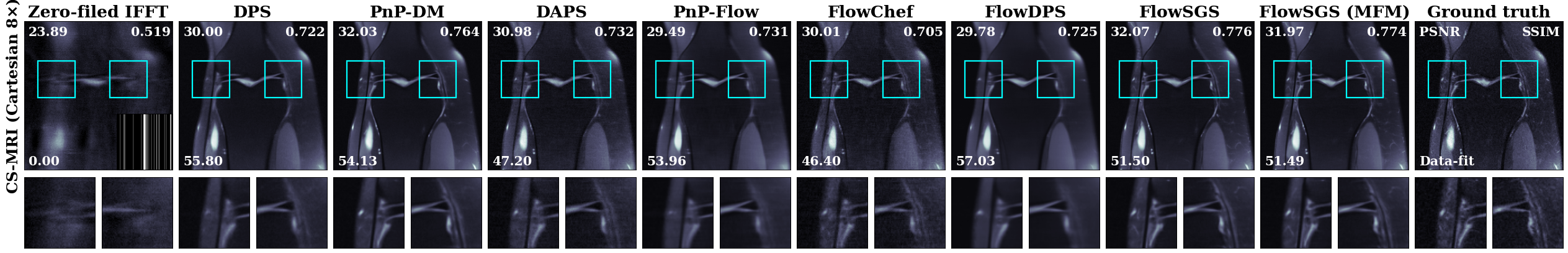}
    \includegraphics[width=\linewidth]{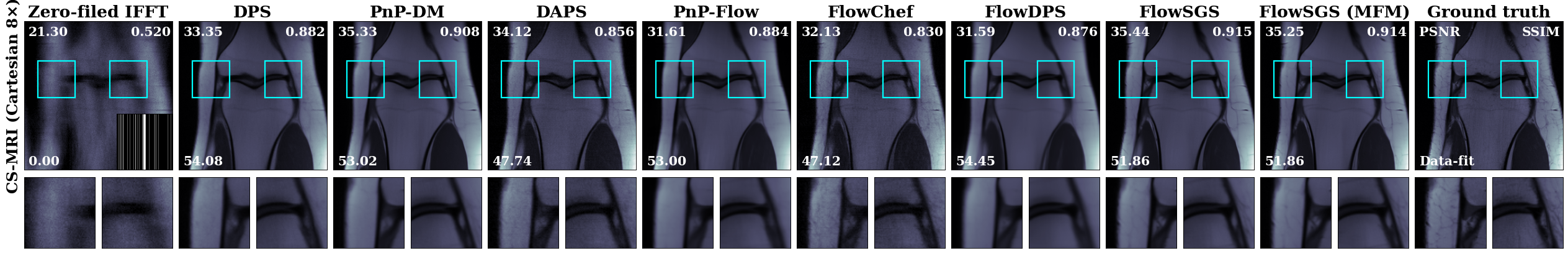}
    \includegraphics[width=\linewidth]{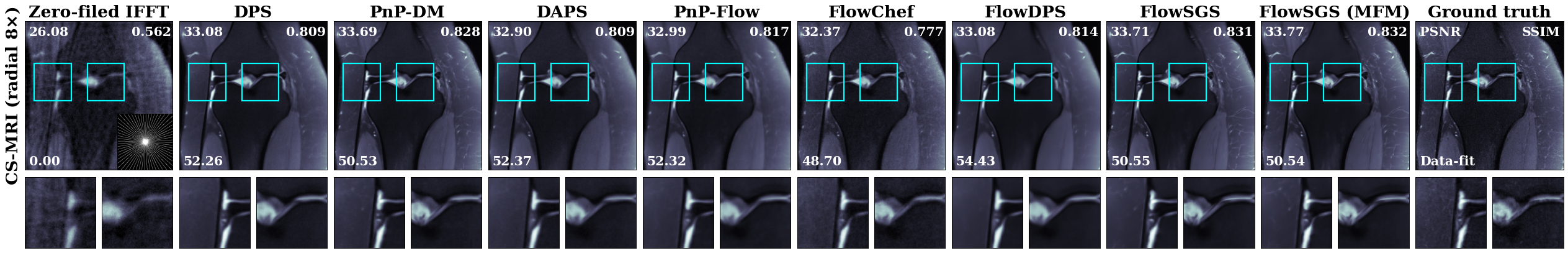}
    \includegraphics[width=\linewidth]{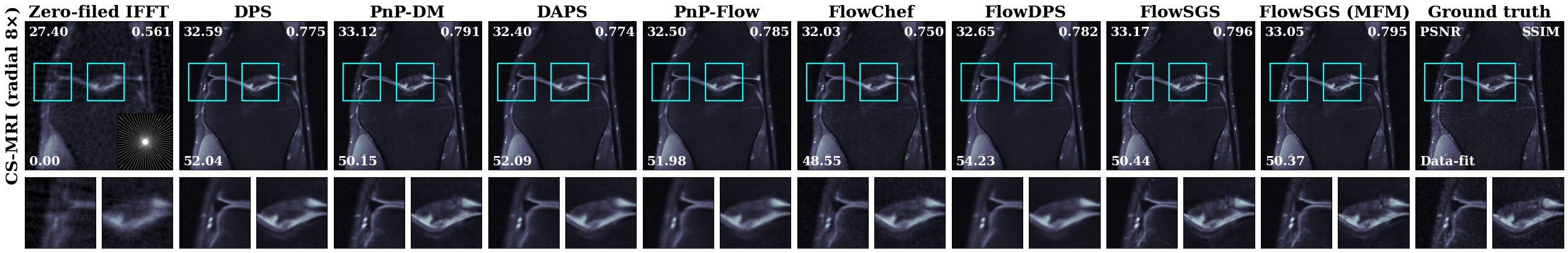}
    \caption{
        \textbf{Additional visual samples on compressed sensing MRI with Cartesian masks and radial masks.}
        We draw 8 posterior samples for each method and show the mean image.
        }
    \label{fig:mri_results_supp}
\end{figure}

\subsection{Main experiments}
\label{app:main}

In this section, we provide additional results in our main experiments using pixel-space models in~\Cref{sec:results_linear,sec:fpr}.
In~\Cref{tab:results_full}, we show the standard deviation of the mean for each metric provided in~\Cref{tab:results}.
We also show the runtime for each method when drawing 8 posterior samples in a batch in~\Cref{tab:time}.
All timing is averaged over 100 runs on a single NVIDIA H100 GPU.
Note that FlowSGS's runtime differs across different interpolant schedules, because different schedules convert to different starting time points for the prior step SDE, leading to different NFEs (881 for linear, 828 for GVP, and 481 for VP-SDE).
With all considered interpolant schedules, FlowSGS is faster than existing diffusion-based samplers with proximal likelihood step.
In~\Cref{fig:pixel_results_supp,fig:mri_results_supp,fig:fpr_app}, we show additional visual samples for the 6 inverse problems in our main experiments in \Cref{sec:results_linear,sec:fpr}.

\begin{table}
    \footnotesize
    \setlength{\tabcolsep}{3pt}
    \centering
    \caption{
        \textbf{Quantitative results averaged over 100 test images with errors.}
        On each of 100 test images, we take the mean across 8 samples for each method and score the mean reconstruction. We report the mean $\pm$ standard deviation of the mean image score for several metrics. For Fourier phase retrieval, performance is computed on the best of each batch of 8 samples rather than their mean.
    }
    \label{tab:results_full}
    \resizebox{\linewidth}{!}{%
    \begin{tabular}{lcccccc}
    \toprule
    \multirow{2}{*}{\bf Method} & \multicolumn{3}{c}{\bf Motion deblurring} & \multicolumn{3}{c}{\bf Gaussian deblurring}\\
    \cmidrule(lr){2-4} \cmidrule(lr){5-7}
    & PSNR ($\uparrow$) & SSIM ($\uparrow$) & LPIPS ($\downarrow$) & PSNR ($\uparrow$) & SSIM ($\uparrow$) & LPIPS ($\downarrow$) \\ \midrule
    DPS      & 26.76 $\pm$ 0.24 & 0.777 $\pm$ 0.007 & 0.184 $\pm$ 0.005 & 28.54 $\pm$ 0.22 & 0.810 $\pm$ 0.006 & 0.151 $\pm$ 0.004 \\
    PnP-DM   & 28.52 $\pm$ 0.21 & 0.822 $\pm$ 0.005 & 0.174 $\pm$ 0.005 & 29.24 $\pm$ 0.21 & 0.832 $\pm$ 0.005 & 0.160 $\pm$ 0.004 \\
    DAPS     & 28.85 $\pm$ 0.21 & 0.813 $\pm$ 0.005 & 0.155 $\pm$ 0.005 & 29.52 $\pm$ 0.21 & 0.822 $\pm$ 0.004 & 0.146 $\pm$ 0.004 \\
    PnP-Flow & 25.42 $\pm$ 0.30 & 0.761 $\pm$ 0.008 & 0.233 $\pm$ 0.007 & 28.34 $\pm$ 0.23 & 0.814 $\pm$ 0.006 & 0.181 $\pm$ 0.005 \\
    FlowChef & 25.18 $\pm$ 0.14 & 0.595 $\pm$ 0.004 & 0.299 $\pm$ 0.006 & 28.31 $\pm$ 0.21 & 0.796 $\pm$ 0.005 & 0.166 $\pm$ 0.004 \\
    FlowDPS  & 28.07 $\pm$ 0.25 & 0.818 $\pm$ 0.006 & 0.182 $\pm$ 0.006 & 29.49 $\pm$ 0.23 & 0.838 $\pm$ 0.005 & 0.158 $\pm$ 0.004 \\ \midrule
    FlowSGS (VP)     & 28.93 $\pm$ 0.22 & 0.840 $\pm$ 0.005 & \underline{0.147 $\pm$ 0.004} & \underline{29.58 $\pm$ 0.22} & \underline{0.844 $\pm$ 0.005} & \underline{0.139 $\pm$ 0.004} \\
    FlowSGS (GVP)    & \underline{29.01 $\pm$ 0.21} & \underline{0.841 $\pm$ 0.005} & 0.149 $\pm$ 0.004 & 29.41 $\pm$ 0.22 & 0.843 $\pm$ 0.005 & 0.144 $\pm$ 0.004 \\
    FlowSGS (Linear) & \textbf{29.23 $\pm$ 0.20} & \textbf{0.843 $\pm$ 0.005} & 0.148 $\pm$ 0.004 & \textbf{29.61 $\pm$ 0.22} & \textbf{0.845 $\pm$ 0.005} & 0.144 $\pm$ 0.004 \\
    FlowSGS (MFM)    & 28.85 $\pm$ 0.21 & 0.832 $\pm$ 0.005 & \textbf{0.139 $\pm$ 0.004} & 29.45 $\pm$ 0.22 & 0.839 $\pm$ 0.005 & \textbf{0.131 $\pm$ 0.004} \\
    \bottomrule
    \end{tabular}}

    \medskip %

    \resizebox{\linewidth}{!}{%
    \begin{tabular}{lcccccc}
    \toprule
    \multirow{2}{*}{\bf Method} & \multicolumn{3}{c}{\bf CS-MRI (Cartesian, $8\times$)} & \multicolumn{3}{c}{\bf CS-MRI (radial, $8\times$)} \\
    \cmidrule(lr){2-4} \cmidrule(lr){5-7}
    & PSNR ($\uparrow$) & SSIM ($\uparrow$) & Data-fit ($\downarrow$) & PSNR ($\uparrow$) & SSIM ($\uparrow$) & Data-fit ($\downarrow$) \\ \midrule
    DPS      & 30.79 $\pm$ 0.18 & 0.802 $\pm$ 0.007 & 55.503 $\pm$ 0.127 & 34.09 $\pm$ 0.20 & 0.850 $\pm$ 0.006 & 52.301 $\pm$ 0.028 \\
    PnP-DM   & \underline{32.43 $\pm$ 0.24} & 0.835 $\pm$ 0.007 & 53.693 $\pm$ 0.065 & 34.77 $\pm$ 0.22 & \underline{0.865 $\pm$ 0.007} & 50.808 $\pm$ 0.032 \\
    DAPS     & 31.92 $\pm$ 0.19 & 0.819 $\pm$ 0.006 & \underline{49.968 $\pm$ 0.030} & 33.79 $\pm$ 0.20 & 0.846 $\pm$ 0.007 & 52.202 $\pm$ 0.068 \\
    PnP-Flow & 29.76 $\pm$ 0.23 & 0.804 $\pm$ 0.007 & 53.980 $\pm$ 0.063 & 33.91 $\pm$ 0.21 & 0.856 $\pm$ 0.006 & 52.477 $\pm$ 0.036 \\
    FlowChef & 29.83 $\pm$ 0.23 & 0.752 $\pm$ 0.006 & \textbf{47.151 $\pm$ 0.027} & 33.30 $\pm$ 0.19 & 0.817 $\pm$ 0.006 & \textbf{48.896 $\pm$ 0.022} \\
    FlowDPS  & 29.09 $\pm$ 0.33 & 0.783 $\pm$ 0.008 & 56.554 $\pm$ 0.162 & 34.02 $\pm$ 0.20 & 0.854 $\pm$ 0.007 & 54.141 $\pm$ 0.079 \\ \midrule
    FlowSGS (VP)     & 31.77 $\pm$ 0.26 & 0.819 $\pm$ 0.006 & 51.732 $\pm$ 0.026 & 34.71 $\pm$ 0.21 & 0.864 $\pm$ 0.006 & \underline{49.758 $\pm$ 0.026} \\
    FlowSGS (GVP)    & 31.99 $\pm$ 0.26 & 0.826 $\pm$ 0.007 & 53.425 $\pm$ 0.039 & \textbf{34.90 $\pm$ 0.22} & \textbf{0.869 $\pm$ 0.006} & 51.157 $\pm$ 0.033 \\
    FlowSGS (Linear) & 32.37 $\pm$ 0.25 & \underline{0.842 $\pm$ 0.007} & 51.930 $\pm$ 0.031 & \underline{34.82 $\pm$ 0.22} & \textbf{0.869 $\pm$ 0.006} & 50.938 $\pm$ 0.034 \\
    FlowSGS (MFM)    & \textbf{32.57 $\pm$ 0.24} & \textbf{0.844 $\pm$ 0.007} & 51.924 $\pm$ 0.030 & 34.81 $\pm$ 0.22 & \textbf{0.869 $\pm$ 0.006} & 50.916 $\pm$ 0.034 \\
    \bottomrule
    \end{tabular}}

    \medskip %

    \resizebox{\linewidth}{!}{%
    \begin{tabular}{lcccccc}
    \toprule
    \multirow{2}{*}{\bf Method} & \multicolumn{3}{c}{\bf Super-resolution ($4\times$)} & \multicolumn{3}{c}{\bf Fourier phase retrieval} \\
    \cmidrule(lr){2-4} \cmidrule(lr){5-7}
    & PSNR ($\uparrow$) & SSIM ($\uparrow$) & LPIPS ($\downarrow$) & PSNR ($\uparrow$) & SSIM ($\uparrow$) & LPIPS ($\downarrow$) \\ \midrule
    DPS      & 28.63 $\pm$ 0.20 & 0.819 $\pm$ 0.005 & \underline{0.153 $\pm$ 0.004} & 15.40 $\pm$ 0.56 & 0.416 $\pm$ 0.018 & 0.484 $\pm$ 0.018 \\
    PnP-DM   & 28.07 $\pm$ 0.18 & 0.811 $\pm$ 0.005 & 0.188 $\pm$ 0.005 & 32.13 $\pm$ 0.99 & 0.901 $\pm$ 0.012 & 0.104 $\pm$ 0.013 \\
    DAPS     & 28.96 $\pm$ 0.18 & 0.815 $\pm$ 0.004 & 0.154 $\pm$ 0.004 & 33.74 $\pm$ 0.24 & 0.899 $\pm$ 0.005 & \underline{0.070 $\pm$ 0.003} \\
    PnP-Flow & 21.92 $\pm$ 0.16 & 0.684 $\pm$ 0.008 & 0.292 $\pm$ 0.006 & 24.63 $\pm$ 0.61 & 0.720 $\pm$ 0.018 & 0.221 $\pm$ 0.014 \\
    FlowChef & 26.91 $\pm$ 0.12 & 0.709 $\pm$ 0.003 & 0.231 $\pm$ 0.005 & 18.24 $\pm$ 0.35 & 0.414 $\pm$ 0.014 & 0.433 $\pm$ 0.010 \\
    FlowDPS  & 28.36 $\pm$ 0.19 & 0.816 $\pm$ 0.005 & 0.186 $\pm$ 0.005 & 27.55 $\pm$ 0.63 & 0.808 $\pm$ 0.015 & 0.183 $\pm$ 0.011 \\ \midrule
    FlowSGS (VP)     & 28.50 $\pm$ 0.16 & 0.824 $\pm$ 0.005 & 0.157 $\pm$ 0.004 & \underline{34.48 $\pm$ 0.90} & \underline{0.925 $\pm$ 0.008} & 0.082 $\pm$ 0.012 \\
    FlowSGS (GVP)    & 28.59 $\pm$ 0.16 & 0.829 $\pm$ 0.005 & 0.161 $\pm$ 0.004 & 31.59 $\pm$ 1.11 & 0.899 $\pm$ 0.014 & 0.128 $\pm$ 0.016 \\
    FlowSGS (Linear) & \textbf{29.09 $\pm$ 0.19} & \underline{0.840 $\pm$ 0.005} & 0.159 $\pm$ 0.004 & \textbf{36.35 $\pm$ 0.90} & \textbf{0.939 $\pm$ 0.010} & \textbf{0.059 $\pm$ 0.011} \\
    FlowSGS (MFM)    & \underline{29.06 $\pm$ 0.19} & \textbf{0.841 $\pm$ 0.004} & \textbf{0.142 $\pm$ 0.004} & 14.88 $\pm$ 0.26 & 0.474 $\pm$ 0.017 & 0.584 $\pm$ 0.011\\
    \bottomrule
    \end{tabular}}
\end{table}

\begin{table}
    \centering
    \caption{
        \textbf{Average sampling runtimes when drawing 8 posterior samples in parallel.}
        All runtimes are measured on a single NVIDIA H100 GPU.
        We report the mean $\pm$ standard deviation of the mean sampling time computed from 100 runs.
    }
    \label{tab:time}
    \begin{tabular}{lc}
       \toprule
        \bf Method & \bf Runtime (s) \\
        \midrule
        DPS & 103.3 $\pm$ 0.0  \\
        PnP-DM (w/ proximal sampling) & 106.3 $\pm$ 0.0 \\
        PnP-DM (w/ Langevin dynamics) & 123.1 $\pm$ 0.0  \\
        DAPS & 64.2 $\pm$ 0.0  \\
        PnP-Flow & 5.4 $\pm$ 0.0 \\
        FlowChef & 5.4 $\pm$ 0.0 $\pm$ 0.0 \\
        FlowDPS & 5.5$\pm$ 0.0  \\
        \midrule
        FlowSGS (Linear, w/ proximal sampling) & 52.9 $\pm$ 0.0 \\
        FlowSGS (Linear, w/ Langevin dynamics) & 144.2 $\pm$ 0.0 \\
        FlowSGS (GVP, w/ proximal sampling) & 41.5 $\pm$ 0.0 \\
        FlowSGS (VP, w/ proximal sampling) & 25.9 $\pm$ 0.0 \\
        FlowSGS (MFM, w/ proximal sampling) & 5.5 $\pm$ 0.0 \\
        \bottomrule
   \end{tabular}
\end{table}

\begin{figure}
    \centering
    \includegraphics[width=\linewidth]{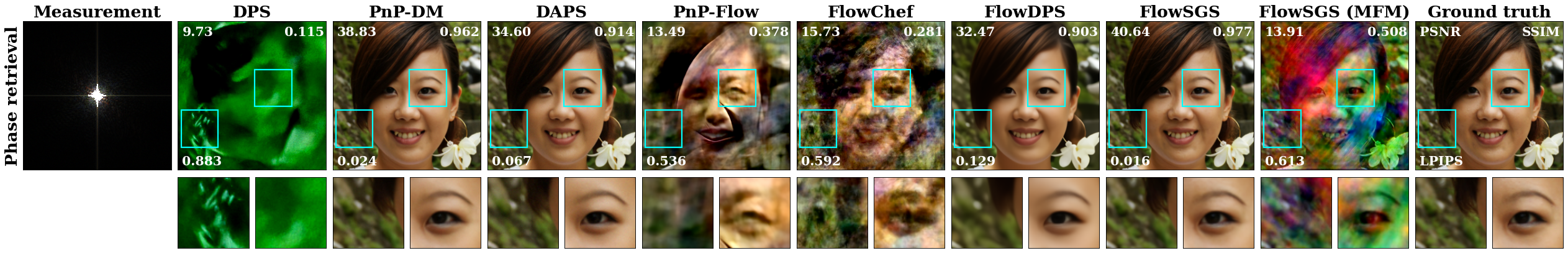}
    \includegraphics[width=\linewidth]{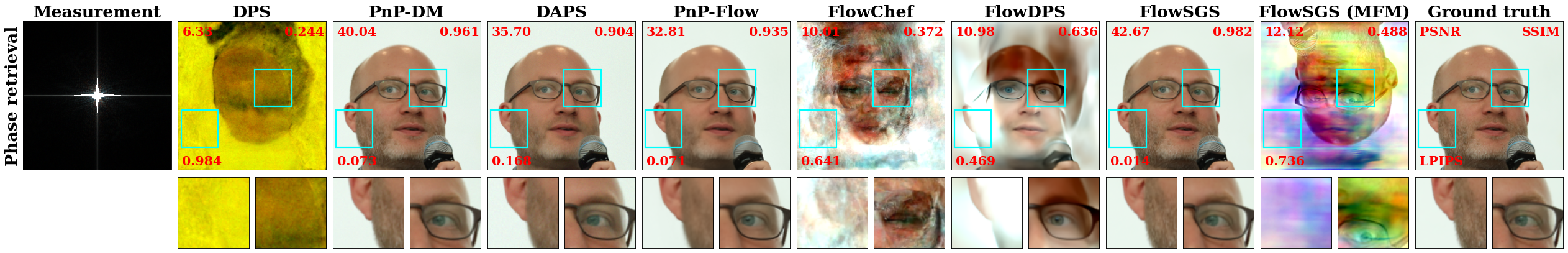}
    \includegraphics[width=\linewidth]{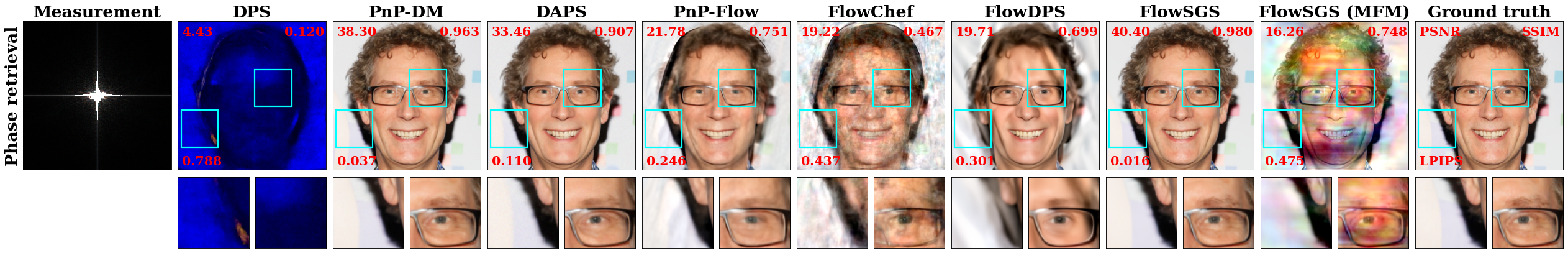}
    \includegraphics[width=\linewidth]{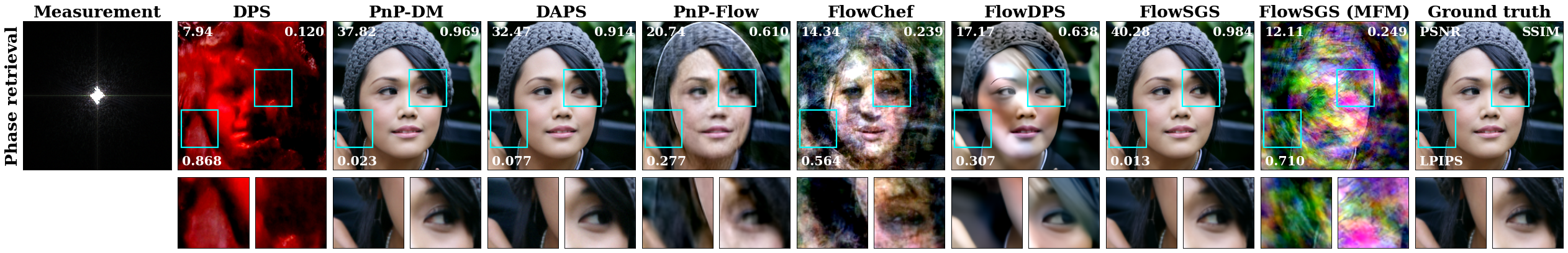}
    \caption{
        \textbf{Additional visual samples on Fourier phase retrieval.}
        We draw 8 posterior samples from each method and report the best one.
        Note that FlowSGS (MFM) results in degraded samples, as discussed in~\Cref{sec:fpr}.
    }
    \label{fig:fpr_app}
\end{figure}

\subsection{Ablation study}
\label{app:ablation}
\paragraph{Key components in FlowSGS}
In~\Cref{tab:ablation}, we test the effectiveness of the key components in FlowSGS by removing them individually.
We show that timestep conversion, scaling, and timestep correction all lead to improved accuracy.
The hybrid prior step acceleration leads to slightly better image quality compared to using MFM in both stages.
Results are averaged over 100 test samples on the motion deblurring task, with errors (standard deviation of the mean) provided.

\paragraph{Interpolant schedules and diffusion coefficients}
In~\Cref{tab:results,tab:ablation_coeffs}, we provide a quantitative comparison of different interpolant coefficients and diffusion coefficients on motion deblurring.
Results are averaged over all 100 test images, with errors (standard deviation of the mean) provided.
We observe that our method is robust across different interpolant schedules and diffusion coefficients.

\begin{table}
    \small
    \centering
    \caption{
        \textbf{Ablation study on key components in the prior step of FlowSGS for motion deblurring}, averaged over 100 test samples in our test set.
        We draw 8 posterior samples for each method and report the metric on the mean image.
    }
    \label{tab:ablation}
    \begin{tabular}{lccc}
   \toprule
   \multirow{2}{*}{\bf Method} & \multicolumn{3}{c}{\bf Motion deblurring} \\
   \cmidrule(lr){2-4}
   & PSNR $\uparrow$ & SSIM $\uparrow$ & LPIPS $\downarrow$ \\ \midrule
   FlowSGS (false conversion) & 23.69 $\pm$ 0.28 & 0.491 $\pm$ 0.014 & 0.311 $\pm$ 0.022 \\
   FlowSGS (no scaling) & 7.04 $\pm$ 0.18 & 0.129 $\pm$ 0.017 & 0.694 $\pm$ 0.018 \\
   FlowSGS (no timestep correction) & 19.40 $\pm$ 0.14 & 0.608 $\pm$ 0.006 & 0.233 $\pm$ 0.006 \\
   FlowSGS & \textbf{29.23 $\pm$ 0.20} & \textbf{0.843 $\pm$ 0.005} & 0.148 $\pm$ 0.004 \\
   \midrule
   FlowSGS (MFM, no hybrid acceleration) & 28.03 $\pm$ 0.19 & 0.813 $\pm$ 0.005 & \underline{0.143 $\pm$ 0.004} \\
   FlowSGS (MFM) & \underline{28.85 $\pm$ 0.21} & \underline{0.832 $\pm$ 0.005} & \textbf{0.139 $\pm$ 0.004} \\
   \bottomrule
\end{tabular}
\end{table}

\begin{table}[!ht]
    \footnotesize
    \centering
    \caption{
        \textbf{Ablation results on diffusion coefficients $w_t$},  averaged over 100 test images on motion deblurring.
        We report the mean $\pm$ standard deviation of the mean image score for several metrics.
        We take the mean across 8 samples for each method and score the mean reconstruction.
    }
    \label{tab:ablation_coeffs}
    \begin{tabular}{lccc}
       \toprule
       \multirow{2}{*}{\bf \makecell[l]{Diffusion \\ Coefficient}}
       & \multicolumn{3}{c}{\bf Motion deblurring}  \\
       \cmidrule(lr){2-4}
       & PSNR\,$\uparrow$ & SSIM\,$\uparrow$ & LPIPS\,$\downarrow$ \\
       \midrule
       $w_t = 0$ (ODE) & 29.22 $\pm$ 0.20 & \underline{0.844 $\pm$ 0.005} & \textbf{0.148 $\pm$ 0.004} \\
       $w_t = \sigma_t$ & \underline{29.32 $\pm$ 0.20} & \textbf{0.845 $\pm$ 0.005} & \underline{0.150 $\pm$ 0.004} \\
       $w_t = \sin^2(\pi t)$ & \textbf{29.36 $\pm$ 0.20} & \underline{0.844 $\pm$ 0.005} & 0.153 $\pm$ 0.005 \\
       $w_t = w_t^{\mathrm{KL}}$ & 29.21 $\pm$ 0.21 & 0.839 $\pm$ 0.005 & 0.161 $\pm$ 0.005 \\
       $w_t = w_t^{\mathrm{KL},\eta}$ & 29.23 $\pm$ 0.20 & 0.843 $\pm$ 0.005 & \textbf{0.148 $\pm$ 0.004} \\
       \bottomrule
   \end{tabular}
\end{table}

\paragraph{Number of SDE discretization steps}
In~\Cref{fig:ablation_steps}, we show visual samples from the ablation study on the SDE discretization steps in~\Cref{fig:ablation}c for motion deblurring.
We see that FlowSGS with timestep correction is more robust under a small discretization step setting, while PnP-DM drastically fails due to high discretization error from the diffusion model.

\begin{figure}
    \centering
    \includegraphics[width=\linewidth]{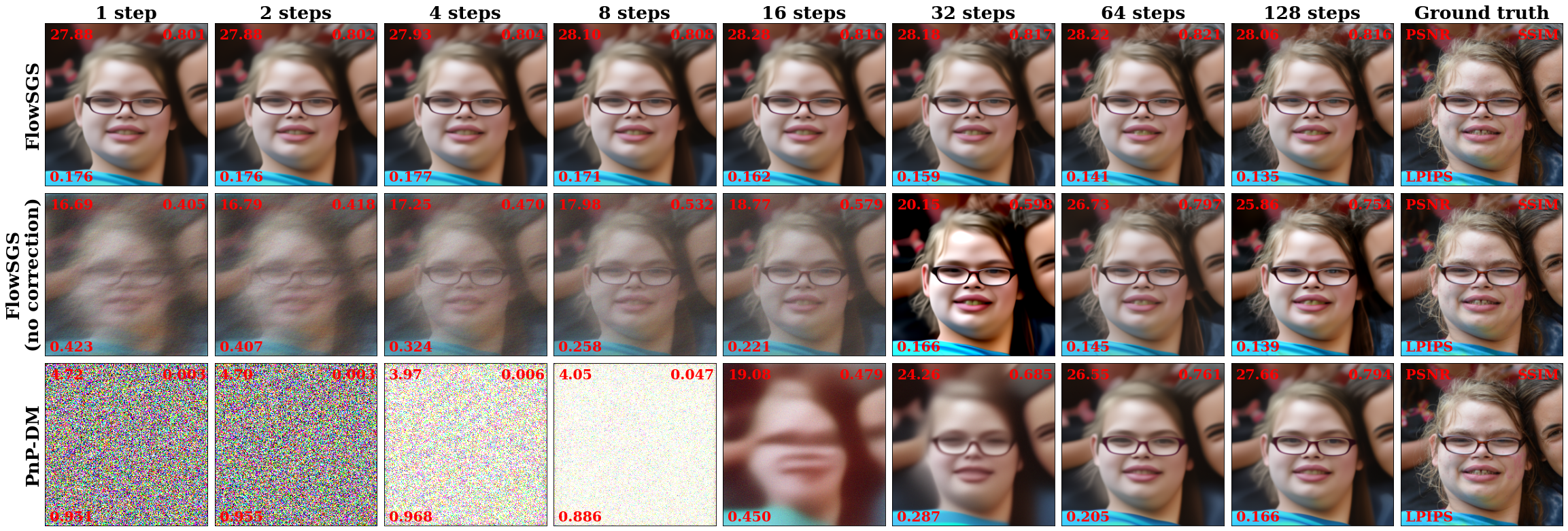}
    \caption{
        \textbf{Visual samples of the ablation study on the SDE discretization steps for motion deblurring.}
        We draw 8 posterior samples for each method and show the mean image.
        }
    \label{fig:ablation_steps}
\end{figure}

\end{document}